\documentclass{article}

\usepackage[utf8]{inputenc} 
\usepackage[T1]{fontenc}    %

\usepackage[colorlinks,linkcolor=linkcolor,citecolor=citecolor]{hyperref}
\usepackage{url}         
\usepackage{amsfonts}       
\usepackage{nicefrac}       %
\usepackage{microtype}
\usepackage{graphicx}
\usepackage{subfigure}
\usepackage{booktabs} 
\usepackage[percent]{overpic}

\usepackage{booktabs}       
\usepackage{amsfonts}       
\usepackage{nicefrac}       
\usepackage{microtype}      
\usepackage{xcolor}

\usepackage{microtype}
\usepackage{subcaption}
\usepackage{graphicx, subfigure}
\usepackage{booktabs} 
\usepackage{xspace}
\usepackage{titletoc}
\usepackage{enumitem}
\usepackage{xfrac}
\usepackage{wrapfig}
\usepackage{bm}
\usepackage{algorithm, algorithmic, caption}
\usepackage{dsfont}
\usepackage{pifont}
\usepackage{adjustbox}
\usepackage[dvipsnames]{xcolor}
\usepackage{multirow}
\usepackage{array}    
\usepackage{makecell} 

\definecolor{Gold}{rgb}{1,0.84,0}
\definecolor{Bronze}{rgb}{0.69,0.55,0.34}
\newcommand{\cmark}{{\color{OliveGreen}\ding{51}}}
\newcommand{\xmark}{{\color{BrickRed}\ding{55}}}

\definecolor{chaptercolor}{HTML}{1A254B}
\definecolor{darkblue}{HTML}{1A254B}
\definecolor{linkcolor}{HTML}{2B50AA}
\definecolor{citecolor}{HTML}{2B50AA}
\definecolor{lightlinkcolor}{HTML}{9A8F97}
\definecolor{darklinkcolor}{HTML}{1A254B}
\definecolor{pink}{HTML}{E05F60}
\definecolor{lightblue}{HTML}{A7BED3}
\definecolor{red}{HTML}{F2545B}
\definecolor{blue}{HTML}{2b50aa}

\usepackage{amsmath}
\usepackage{amssymb}
\usepackage{mathtools}
\usepackage{amsthm}
\usepackage{bbm}
\usepackage{nicefrac}
\usepackage[capitalize,noabbrev]{cleveref}

\usepackage[para,online,flushleft]{threeparttable}
\usepackage{adjustbox}
\theoremstyle{plain}
\newtheorem{theorem}{Theorem}[section]

\newtheorem{lemma}[theorem]{Lemma}
\newtheorem{corollary}[theorem]{Corollary}
\theoremstyle{definition}
\newtheorem{definition}[theorem]{Definition}
\newtheorem{assumption}[theorem]{Assumption}

\theoremstyle{remark}

\Crefname{assumption}{Assumption}{Assumptions}
\crefname{assumption}{Assumption}{Assumptions}
\allowdisplaybreaks

\makeatletter
\renewcommand{\paragraph}{%
  \@startsection{paragraph}{4}%
  {\z@}{0ex \@plus 0ex \@minus 0ex}{-1em}%
  {\normalfont\normalsize\bfseries}
}
\makeatother

\newlist{lemenum}{enumerate}{1} 
\setlist[lemenum]{label=(\roman*),ref=\thelemma\,(\roman*),topsep=0pt}
\Crefname{lemenumi}{Lemma}{Lemmas}

\newlist{corenum}{enumerate}{1} 
\setlist[corenum]{label=(\roman*),ref=\thecorollary\,(\roman*),topsep=0pt}
\Crefname{corenumi}{Corollary}{Corollaries}

\newlist{thmenum}{enumerate}{1} 
\setlist[thmenum]{label=(\roman*),ref=\thetheorem\,(\roman*),topsep=0pt}
\Crefname{thmenumi}{Theorem}{Theorems}

\newlist{propenum}{enumerate}{1} 
\setlist[propenum]{label=(\roman*),ref=\thedefinition\,(\roman*),topsep=0pt}
\Crefname{propenumi}{Property}{Properties}

\newlist{assenum}{enumerate}{1} 
\setlist[assenum]{label=(\roman*),ref=\theassumption\,(\roman*),topsep=0pt}
\Crefname{assenumi}{Assumption}{Assumptions}

\usepackage{svg}
\usepackage{graphicx}
\usepackage{rotating}
\usepackage{tikz}
\usepackage{pgfplots}
\pgfplotsset{compat=newest}
\usetikzlibrary{shapes.geometric}

\usepackage{import}
\usepackage{xifthen}
\usepackage{pdfpages}
\usepackage{transparent}

\NewDocumentCommand{\incfig}{mo}{
  \begin{center}
    \IfValueT{#2}{\def\svgwidth{#2}}{\def\svgwidth{\columnwidth}}
    \import{./figures/}{#1.pdf_tex}
  \end{center}
}

\usepackage{pgf}
\usepackage{adjustbox}

\NewDocumentCommand{\incplt}{O{\columnwidth}m}{%
  \begin{center}
    \adjustbox{width=#1}{\import{./plots/output/}{#2.pgf}}
  \end{center}
}



\usepackage{aligned-overset}

\newcommand{\neorl}{\textcolor{black}{\textsc{NeoRL}}\xspace}
\newcommand{\ombrl}{\textcolor{black}{\textsc{SOMBRL}}\xspace}

\DeclareFontFamily{U}{mathb}{\hyphenchar\font45}
\DeclareFontShape{U}{mathb}{m}{n}{
      <5> <6> <7> <8> <9> <10> gen * mathb
      <10.95> mathb10 <12> <14.4> <17.28> <20.74> <24.88> mathb12
      }{}
\DeclareSymbolFont{mathb}{U}{mathb}{m}{n}
\DeclareFontSubstitution{U}{mathb}{m}{n}
\DeclareMathSymbol{\Asterisk}      {2}{mathb}{"06}
\newcommand{\mainsym}{\textcolor{blue}{\parentheses*{\hspace{-0.1em}\Asterisk\hspace{-0.1em}}}}

\newcommand{\psym}{\ensuremath{\textcolor{blue}{\dagger}}\xspace}

\newcommand*{\abs}[1]{| #1 |}

\NewDocumentCommand{\norm}{sm}{\IfBooleanTF{#1}{\|#2\|}{\left\| #2 \right\|}}
\newcommand{\setmath}[1]{\left\{#1\right\}}

\DeclareMathOperator*{\defeq}{\smash{\overset{\mathrm{def}}{=}}}

\DeclareMathOperator*{\argmax}{arg\,max}
\DeclareMathOperator*{\argmin}{arg\,min}

\DeclarePairedDelimiter\parentheses{(}{)}
\DeclarePairedDelimiter\brackets{[}{]}

\newcommand{\R}{\mathbb{R}}
\newcommand{\E}{\mathbb{E}}
\newcommand{\Rzero}{\mathbb{R}_{\geq 0}}

\renewcommand{\vec}[1]{{\bm{#1}}}
\newcommand{\mat}[1]{\bm{#1}}

\NewDocumentCommand{\fnPr}{}{\mathbb{P}}
\RenewDocumentCommand{\Pr}{om}{\fnPr\IfValueT{#1}{_{#1}}\parentheses*{#2}}
\NewDocumentCommand{\Prsm}{om}{\fnPr\IfValueT{#1}{_{#1}}\parentheses{#2}}
\RenewDocumentCommand{\H}{mo}{\mathrm{H}\IfValueTF{#2}{\!\left[#1\ \middle|\ #2\right]}{\brackets*{#1}}}
\NewDocumentCommand{\Hsm}{mo}{\mathrm{H}\IfValueTF{#2}{[#1 \mid #2]}{\brackets{#1}}}
\NewDocumentCommand{\I}{mmo}{\mathrm{I}\IfValueTF{#3}{\!\left(#1;#2\ \middle|\ #3\right)}{\parentheses*{#1; #2}}}
\NewDocumentCommand{\Ism}{mmo}{\mathrm{I}\IfValueTF{#3}{(#1;#2 \mid #3)}{\parentheses{#1; #2}}}

\NewDocumentCommand{\ExpVal}{somo}{\ensuremath{\mathbb{E}\IfValueT{#2}{_{#2}}{} \IfBooleanTF{#1}{#3}{\IfValueTF{#4}{\!\left[#3\ \middle|\ #4\right]}{\brackets*{#3}}}}}
\NewDocumentCommand{\Esm}{somo}{\ensuremath{\mathbb{E}\IfValueT{#2}{_{#2}}{} \IfBooleanTF{#1}{#3}{\IfValueTF{#4}{\!\left[#3\ \middle|\ #4\right]}{\brackets{#3}}}}}
\NewDocumentCommand{\Var}{somo}{\mathrm{Var}\IfValueT{#2}{_{#2}}{} \IfBooleanTF{#1}{#3}{\IfValueTF{#4}{\!\left[#3\ \middle|\ #4\right]}{\brackets*{#3}}}}
\NewDocumentCommand{\Varsm}{somo}{\mathrm{Var}\IfValueT{#2}{_{#2}}{} \IfBooleanTF{#1}{#3}{\IfValueTF{#4}{\left[#3\ \middle|\ #4\right]}{\brackets{#3}}}}
\NewDocumentCommand{\Cov}{som}{\mathrm{Cov}\IfValueT{#2}{_{#2}}{} \IfBooleanTF{#1}{#3}{\brackets*{#3}}}
\NewDocumentCommand{\Cor}{som}{\mathrm{Cor}\IfValueT{#2}{_{#2}}{} \IfBooleanTF{#1}{#3}{\brackets*{#3}}}

\NewDocumentCommand{\grad}{e_}{\bm{\nabla}\IfValueT{#1}{_{\!\!#1}\,}}

\RenewDocumentCommand{\det}{m}{\left| #1 \right|}

\NewDocumentCommand{\tr}{m}{\mathrm{tr}\;#1}
\NewDocumentCommand{\diag}{som}{\mathrm{diag}\IfValueT{#2}{_{#2}}{}\,#3}

\NewDocumentCommand{\N}{somm}{\mathcal{N}\IfBooleanTF{#1}{\left(}{(}\IfValueT{#2}{#2;}{} #3, #4\IfBooleanTF{#1}{\right)}{)}}
\NewDocumentCommand{\GP}{omm}{\mathcal{GP}(\IfValueT{#1}{#1;}{} #2, #3)}

\newcommand{\vzero}{\vec{0}}

\newcommand{\va}{\vec{a}}

\newcommand{\vc}{\vec{c}}

\newcommand{\vf}{\vec{f}}

\newcommand{\vk}{\vec{k}}

\newcommand{\vm}{\vec{m}}

\newcommand{\vu}{\vec{u}}

\newcommand{\vw}{\vec{w}}

\newcommand{\vx}{\vec{x}}

\newcommand{\vs}{\vec{s}}

\newcommand{\vy}{\vec{y}}

\newcommand{\vz}{\vec{z}}

\newcommand{\vmu}{\bm{\mu}}

\newcommand{\vpi}{\bm{\pi}}

\newcommand{\vsigma}{\bm{\sigma}}

\newcommand{\mI}{\mat{I}}
\newcommand{\mM}{\mat{M}}

\newcommand{\mK}{\mat{K}}

\newcommand{\mV}{\mat{V}}

\def\setA{{\mathcal{A}}}

\def\setD{{\mathcal{D}}}

\def\setH{{\mathcal{H}}}

\def\setK{{\mathcal{K}}}

\def\setM{{\mathcal{M}}}
\def\setN{{\mathcal{N}}}
\def\setO{{\mathcal{O}}}

\def\setU{{\mathcal{U}}}

\def\setW{{\mathcal{W}}}
\def\setX{{\mathcal{X}}}

\def\setZ{{\mathcal{Z}}}
\DeclareMathOperator{\determinant}{det}
\def\valpha{{\bm{\alpha}}}
\newcommand{\inner}[2]{\left\langle#1, #2\right\rangle}
\DeclarePairedDelimiter{\ceil}{\lceil}{\rceil}

\usepackage[final]{neurips_2025}

\usepackage{amsmath}
\usepackage{amssymb}
\usepackage{mathtools}
\usepackage{amsthm}

\usepackage[capitalize,noabbrev]{cleveref}

\usepackage[textsize=tiny]{todonotes}
\usepackage{parskip}

\title{\ombrl: Scalable and Optimistic Model-Based RL}
\author{%
  Bhavya Sukhija \\
  Department of Computer Science\\
  ETH Zurich\\
  \texttt{sukhijab@ethz.ch} \\
    \And
Lenart Treven \\
  Department of Computer Science\\
  ETH Zurich\\
  \texttt{trevenl@ethz.ch} \\
  \And
  Carmelo Sferrazza \\
  Berkeley AI Research\\
  UC Berkeley\\
  \texttt{csferrazza@berkeley.edu} \\
  \And
  Florian Dörfler \\
  Department of Electrical Engineering\\
  ETH Zurich\\
  \texttt{dorfler@ethz.ch} \\
  \And
  Pieter Abbeel \\
  Berkeley AI Research\\
  UC Berkeley\\
  \texttt{pabbeel@berekeley.edu} \\
  \And
  Andreas Krause \\
  Department of Computer Science\\
  ETH Zurich\\
  \texttt{krausea@ethz.ch} \\
  }

\begin{document}
\maketitle








\begin{abstract}
\looseness=-1
We address the challenge of efficient exploration in model-based reinforcement learning (MBRL), where the system dynamics are unknown and the RL agent must learn directly from online interactions. 
We propose \textbf{S}calable and \textbf{O}ptimistic \textbf{MBRL} (\ombrl), 
an approach based on the principle of optimism in the face of uncertainty. \ombrl learns an uncertainty-aware dynamics model and \emph{greedily} maximizes a weighted sum of the extrinsic reward and the agent's epistemic uncertainty.  \ombrl is compatible with any policy optimizers or planners, and 
under common regularity assumptions on the system, we show that \ombrl has sublinear regret for nonlinear dynamics in the (\emph{i}) finite-horizon, (\emph{ii}) discounted infinite-horizon, and (\emph{iii}) non-episodic settings.
Additionally, \ombrl offers a flexible and scalable solution for principled exploration.  We evaluate \ombrl on state-based and visual-control environments, where it displays strong performance across 
all tasks and baselines.  We also evaluate \ombrl on a dynamic RC car hardware and show \ombrl outperforms the state-of-the-art, illustrating the benefits of principled exploration for MBRL.
\end{abstract}

\section{Introduction}\label{sec: intro}
\begin{figure*}[ht]
    \centering
    \includegraphics[width=\linewidth]{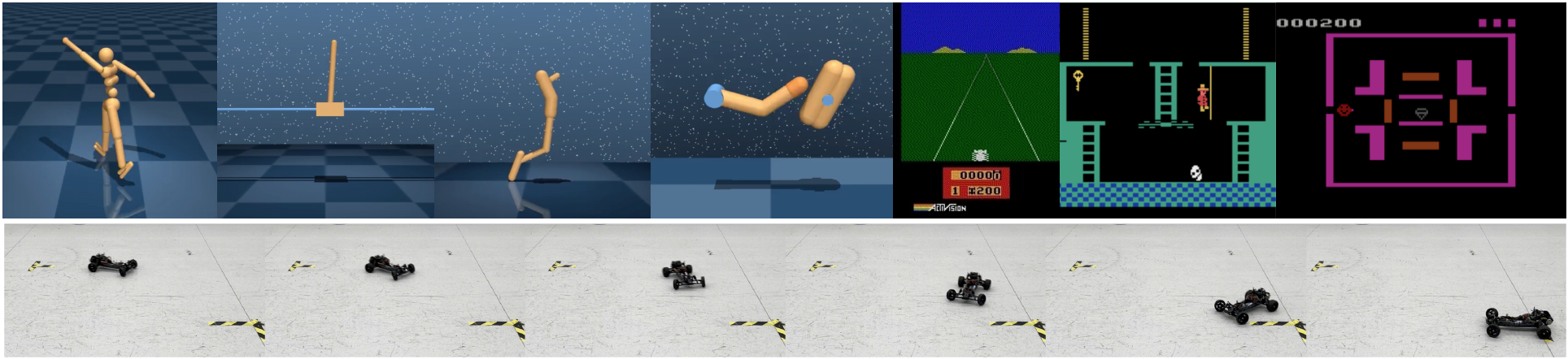}
    \caption{
    \emph{Top:} We showcase scalability of \ombrl on visual control tasks from DMC and Atari.
    \emph{Bottom:} We evaluate \ombrl on a highly dynamic RC car where we learn to perform a complex parking maneuver in only 20 real-world episodes.
    }
    \label{fig:rccar_teaser}
     \vspace{-0.5em}
\end{figure*}
\looseness -1 Reinforcement learning (RL) has been successfully applied to a variety of sequential-decision making problems like games~\citep{silver2017mastering}, robotics~\citep{kober2013reinforcement, tang2025deep}, mobile health interventions~\citep{yom2017encouraging, liao2020personalized}, 
and fine-tuning of large language models~\citep{ouyang2022training}. RL offers a flexible learning paradigm, enabling agents to learn directly by interacting with their environment. However, 
this potential is often not fully realized in practice, as most widely used RL methods~\citep{schulman2017proximal} are highly sample-inefficient. This 
mostly rules out their direct application to real-world settings where data is scarce or expensive to acquire. 

Model-based RL approaches~\citep{moerland2023model} offer a more sample-efficient alternative and have been successfully used for learning directly in the real-world~\citep{hansen2022modem, wu2023daydreamer, rothfuss2024bridging}. However, these methods are mostly based on naive exploration strategies, such as Boltzmann exploration, which are provably sub-optimal~\citep {cesa2017boltzmann} and often struggle in the presence of sparse rewards. 

\paragraph{Related Works} 
\looseness=-1
Several works study principled exploration approaches in RL (\citet{even2001convergence, jaksch10a, 
abbasi2011regret, 
cohen2019learning,
dean2020sample,kakade2020information, curi2020efficient,neu2020unifying, foster2023foundations,
wagenmaker2023optimal, sukhija2024optimistic}, see \cref{sec: exploration strategies in MBRL} and \cref{sec: related works} for more details). In particular, optimism in the face of uncertainty is a celebrated exploration principle with strong theoretical guarantees for model-based RL~\citep{brafman2002r,jaksch10a, kakade2020information, curi2020efficient, moulin2023optimistic, sukhija2024neorl}. However, in practice, these algorithms are computationally prohibitive. As a result, naive exploration techniques remain dominant in real-world applications due to their simplicity. 
We address this gap between theory and practice and propose a simple yet principled method, \ombrl, for exploration that enjoys convergence guarantees across several RL settings. 
This is in contrast to prior works, which design and study algorithms only for specific settings, e.g., \citet{curi2020efficient, kakade2002approximately} study the finite-horizon setting, \citet{sukhija2024optimistic} the unsupervised RL setting, and \citet{sukhija2024neorl} the nonepisodic one. In addition to the theoretical guarantees, \ombrl is also much more scalable and, unlike the aforementioned works, can be applied to high-dimensional settings such as visual control problems and real-world hardware platforms. We demonstrate this in our experiments (\cref{sec: experiments}).

The differences between \ombrl and prior work on exploration in MBRL are summarized in \cref{tab:prior_work}, and we provide a more detailed discussion of related works in \cref{sec: related works}. 

\paragraph{Contributions}
\begin{enumerate}[leftmargin=0.5cm]
 \item  We propose \ombrl, a principled yet efficient exploration strategy for model-based RL. \ombrl is based on the principle of optimism in the face of uncertainty, and \emph{greedily} maximizes 
  a weighted sum of the extrinsic reward and the agent's epistemic uncertainty/disagreement. Therefore, the agent selects policies that maximize rewards while also exploring less visited areas of the state space where the uncertainty about the system is high.
  
 \item \looseness=-1 We show, under common regularity assumptions on the dynamics,
 that combining extrinsic rewards with the agent's epistemic uncertainty gives anytime high probability value-function bounds, which could be of independent interest to applications such as safe RL~\citep{brunke2022safe} and offline RL~\citep{levine2020offline}. We leverage this key insight and show that \ombrl has sublinear regret for finite-horizon, discounted infinite-horizon, and nonepisodic settings with continuous state and action spaces. Our regret bounds are comparable to the ones derived by prior work~\citep{kakade2020information, curi2020efficient, sukhija2024neorl}, but our algorithm is considerably simpler and more scalable. 
 \item  
 We validate \ombrl on standard deep RL benchmarks, showing that it outperforms several naive exploration baselines and scales effectively to high-dimensional tasks, such as visual control. We also \textit{evaluate \ombrl in the real-world} on a dynamic RC car (see~\cref{fig:rccar_teaser}), where it learns an agile parking maneuver in only 20 trials, outperforming the state-of-the-art~\citep{rothfuss2024bridging} w.r.t.~performance and sample efficiency. To the best of our knowledge, this is the first empirical demonstration of optimistic exploration in model-based deep RL for high-dimensional and real-world settings. 
 \end{enumerate}

\begin{table}[t]
\newcommand{\extension}{\xmark\tnote{1}}
\centering
\caption{\small Comparison of \ombrl and prior work on model-based RL (MBRL) in terms of sublinear regret/sample complexity for the kernelized settings and scalability. \citet{kakade2020information} give regret bound for the finite-horizon setting with Gaussian noise and \citet{curi2020efficient} for the sub-Gaussian noise case. We show that \ombrl has sublinear regret for both settings (\cref{thm: finite horizon regret} and \cref{thm: finite horizon regret sub Gaussian})} 
\begin{adjustbox}{max width=\linewidth}\begin{threeparttable}
    \begin{tabular}{cccccc}
       \toprule
       \multirow{2}{*}
        & Finite Horizon &
        $\gamma$-discounted infinite-horizon    &
        Nonepisodic  & Unsupervised RL &
        Scalable + Practical \\
        \midrule
        Greedy, e.g.,  Mean planning or & \multirow{2}{*}{\xmark} & \multirow{2}{*}{\xmark} & \multirow{2}{*}{\xmark} & \multirow{2}{*}{\xmark} & \multirow{2}{*}{\cmark} \\ 
    \citet{deisenroth2011pilco, chua2018pets}  \\ \\
    \citet{curi2020efficient, kakade2020information} & \cmark & \xmark & \xmark & \xmark & \xmark \\ \\
     \citet{sukhija2024neorl} & \xmark & \xmark & \cmark & \xmark & \xmark \\ \\
    \citet{sukhija2024optimistic} & \xmark & \xmark & \xmark & \cmark & \xmark \\ \\
        \textbf{\ombrl (ours)} & \cmark & \cmark & \cmark & \cmark & \cmark \\
       \bottomrule
    \end{tabular}
\end{threeparttable}\end{adjustbox}
\label{tab:prior_work}
 \vspace{-1.5em}
\end{table}

\section{Problem Setting}
We consider a discrete-time dynamical system of the form
$\vx_{t+1} = \vf^*(\vx_t, \vu_t) + \vw_t$, where $\vx_t \in \setX \subseteq \R^{d_\vx}$ is the state, $\vu_t \in \setU \subseteq \R^{d_\vu}$ the control input, and $\vw_t \in \setW \subseteq \R^{\vw}$ the process noise\footnote{For our theory, we assume the process noise to be known, but our algorithm can learn it from data.}.  
The dynamics $\vf^*$ are unknown.  
 \paragraph{Task} In the finite-horizon RL setting~\citep{puterman2014markov}, we are given a reward function $r: \setX \times \setU \to \R$, and want to learn a policy that maximizes the following objective
\begin{equation}
J(\vpi^{*}) = \max_{\vpi \in \Pi} J(\vpi) = \max_{\vpi \in \Pi} \E_{\vpi} \left[ \sum^{T-1}_{t=0} r(\vx_t, \vu_t) \right],
  \label{eq:finite horizon setting}
\end{equation}
where action $\vu_t$ follows policy $\vpi$, i.e., $\vu_t \sim \vpi(\vx_t)$.
Moreover, we consider the episodic RL setting, with episodes $n \in \setmath{1, \ldots, N}$, and study a model-based approach.
Accordingly, at the beginning of episode $n$, we select and roll out a policy $\vpi_n$ for $T$ steps on the true system.  
We then use the data collected from the rollouts to estimate the true dynamics $\vf^*$.
The goal is to find a policy that performs as well as $\vpi^*$, as quickly as possible.
Therefore a natural performance metric in this context is the {\em cumulative regret} $R_N = \sum^N_{n=1} J(\vpi^{*}) - J(\vpi_n)$.
In the following sections, we show that our proposed algorithm achieves sublinear regret. While in the main text for clarity we focus on the finite-horizon episodic setting, 
in \cref{subsec: theory main}, we show that our approach has sublinear regret also for 

\begin{minipage}[t]{0.45\textwidth}
$\gamma$-discounted infinite-horizon, episodic: 
    \begin{equation}
        J_{\gamma}(\vpi^{*}) = \max_{\vpi \in \Pi} \E_{\vpi} \left[ \sum^{\infty}_{t=0} \gamma^t r(\vx_t, \vu_t) \right]
        \label{eq: discounted}
    \end{equation}
\end{minipage}
\hfill
\begin{minipage}[t]{0.54\textwidth}
    and average reward, nonepisodic settings:
    \begin{equation}
        J_{\text{avg}}(\vpi^{*}) = \max_{\vpi \in \Pi} \limsup_{T \to \infty} \frac{1}{T} \E_{\vpi} \left[ \sum^{T-1}_{t=0} r(\vx_t, \vu_t) \right]
        \label{eq: average reward}
    \end{equation}
\end{minipage}


    
 \section{Exploration Strategies in MBRL} 
 \label{sec: exploration strategies in MBRL}
In MBRL, we learn a model of the true dynamics $\vf^*$ and use our learned model to select/update the next policy for data acquisition. Exploration algorithms for MBRL determine how the policy should be chosen given our learned model.
Common strategies for this choice are (\emph{i}) greedy planning, (\emph{ii}) Thompson sampling, and (\emph{iii}) optimistic exploration. We discuss these in detail below.

Let $J(\vpi, \vf)$ be the the expected returns under the policy $\vpi$ and dynamics $\vf$, that is
\begin{equation*}
J(\vpi, \vf) = \E_{\vpi} \left[ \sum^{T-1}_{t=0} r(\vx'_t, \vu_t)\right], \; 
  \vx'_{t+1} = \vf(\vx'_t, \vu_t) + \vw_t, \vx'_0 = \vx_0 \notag,
\end{equation*} 
and $\vmu_n$ our mean estimate of the dynamics $\vf^*$ at episode~$n$.

\paragraph{Greedy planning} The simplest selection strategy is to pick the policy $\vpi_n$ that maximizes the expected returns for our estimated dynamics $\vmu_n$.
\begin{equation}
   \vpi^{\textsc{Mean}}_n = \underset{\vpi \in \Pi}{\arg\max}\; J(\vpi, \vmu_n) 
   \label{eq: mean sampling}
\end{equation}
\looseness=-1
This strategy is greedy as it does not directly encourage exploration in areas where we have limited data or where our model has high uncertainty. Instead, it exploits our estimate $\vmu_n$ of the dynamics.
This is the basis of methods such as those of \citet{janner2019trust, hafner2023mastering}, where exploration is induced using a stochastic policy that is optimized with an entropy bonus. 

\looseness -1 To incorporate epistemic uncertainty in our learned model and avoid overfitting to misestimated dynamics,  \citet{deisenroth2011pilco, chua2018pets, rothfuss2024bridging} learn a Bayesian model of $\vf^*$: $p(\vf|\setD_{1:n})$. Here $\setD_{1:n} = \cup_{i \le n} \setD_{i}$, and $\setD_i = \{(\vx_{t, i}, \vu_{t, i}, \vx_{t+1, i})\}_{t=0}^{T-1}$ is the data collected in episode $i$. The policy $\vpi_n$ is then selected as
\begin{equation}
   \vpi^{\textsc{Greedy}}_n = \underset{\vpi \in \Pi}{\arg\max}\; \E_{\vf \sim p(\vf|\setD_{1:n})}[J(\vpi, \vf)].
   \label{eq: greedy sampling}
\end{equation}

\looseness=-1
\citet{curi2020efficient} show that greedy planning may fail to perform well in practice, especially for difficult exploration problems (e.g., in context of action penalties).

\paragraph{Thompson Sampling} In Thompson sampling (TS), we also learn a Bayesian model $p(\vf|\setD_{1:n})$ and pick policies by maximizing the reward under $\vf$ sampled from the posterior
\begin{equation}
   \vpi^{\textsc{TS}}_n = \underset{\vpi \in \Pi}{\arg\max}\; J(\vpi, \vf), \; \vf \sim p(\vf|\setD_{1:n}).
   \label{eq: thompson sampling}
\end{equation}

\looseness=-1
While TS encourages exploration in a theoretically grounded manner~\citep{russo2018tutorial}, in practice, it is often intractable to sample a function $\vf$ from $p(\vf|\setD_{1:n})$.

\paragraph{Optimistic Exploration}
This strategy is based on the principle of optimism in the face of uncertainty. Optimistic exploration approaches maintain a set of {\em plausible dynamics models} $\setM_n$ at each episode $n$, e.g., the set of functions that have a high probability w.r.t.~a learned Bayesian model $p(\vf|\setD_{1:n})$. The policy is then selected according to
\begin{equation}
 \vpi_n^{\text{OE}} = \underset{\vpi \in \Pi, \vf \in \setM_n}{\arg\max}\;  J(\vpi, \vf)
  \label{eq:optimistic plan expensive}
\end{equation}
\looseness=-1
There are several works that study optimistic exploration theoretically~\citep{jaksch10a, kakade2020information, curi2020efficient, treven2024ocorl, sukhija2024neorl}. However,
optimizing $\vf$ over $\setM_n$, typically a difficult non-convex constraint, is often computationally prohibitive, restricting the application of these methods to fairly low-dimensional settings. The most efficient solvers of the optimization problem \eqref{eq:optimistic plan expensive}, to the best of our knowledge, are based on a reparametrization trick which introduces additional hallucinated controls~\citep{curi2020efficient}. This increases the total control dimension from $d_{\vu}$ to $d_{\vu} + d_{\vx}$, which is prohibitive in high-dimensional domains.

\section{\ombrl: Scalable and Optimistic MBRL} \label{sec: method}
\looseness=-1
We now present \ombrl, our approach for efficient optimistic exploration in MBRL, which alternates between two steps. First, given a dataset of transitions $\setD_{1:n}$, we learn an uncertainty-aware model of the unknown dynamics $\vf^*$. That is, after each episode $n$, we learn a mean estimate $\vmu_n$ of $\vf^*$ and quantify our epistemic uncertainty $\vsigma_n$ over the estimate. Models such as Gaussian processes (GPs)~\citep{rasmussen2005gp} can be directly used for this purpose.
Bayesian deep learning approaches such as deep ensembles are also commonly used to quantify epistemic uncertainty or model disagreement in RL~\citep{
chua2018pets,pathak2019self,
curi2020efficient, sekar2020planning,sukhija2024optimistic}.
In the second step, we solve the following optimization problem for the policy $\vpi_n$
\begin{align}
\vpi_n\! 
&:= \underset{\vpi \in \Pi}{\arg\max}\; \underbrace{\!\E_{\vpi} \!\!\left[ \sum^{T-1}_{t=0} r(\vx'_t, \vu_t) \!+\! \lambda_n \!\norm{\vsigma_n(\vx'_t, \vu_t)}\!\right]}_{J_n(\vpi)},
\; \vx'_{t+1}\!\! = \!\vmu_n(\vx'_t, \vu_t) + \vw_t,
  \label{eq:optimistic plan} 
\end{align}
where $\lambda_n$ is a positive constant which is used to trade off maximizing the extrinsic reward and model uncertainty (see \cref{appendix: theory} for how $\lambda_n$ is defined in theory and \cref{sec: selecting lambda} and \cref{appendix: experiment_details} for how it is selected empirically).
Note that in \cref{eq:optimistic plan}, we use the mean dynamics for planning and only use the epistemic uncertainty as an additional \emph{intrinsic} reward. Compared to the principled exploration strategies from \cref{sec: exploration strategies in MBRL}, our approach does not require sampling from or maximizing over the dynamics.
This makes \ombrl much simpler and more scalable.
Moreover, \ombrl can be combined with any model-based algorithm such as those of~\citet{janner2019trust, hafner2023mastering, rothfuss2024bridging}. The only additional modification we make to these methods is that we add the epistemic uncertainty to the extrinsic reward. Also note that without the epistemic uncertainty reward, i.e., $\lambda_n = 0$, the agent follows the greedy strategy discussed in \cref{sec: exploration strategies in MBRL} and for $\lambda \to \infty$ the agent performs unsupervised exploration~\citep{pathak2017curiosity, sekar2020planning, buisson2020actively, sukhija2024optimistic}. Therefore, we use the model uncertainty to facilitate principled exploration for the agent. 



 In the following, we show that by optimizing our objective in \cref{eq:optimistic plan}, we are effectively maximizing an optimistic estimate of $J(\vpi^*)$, i.e., we are also performing optimistic exploration. Accordingly, our approach enjoys the same guarantees as other optimistic MBRL algorithms but is much simpler, computationally cheaper, and scalable to high-dimensional settings.

\section{Theoretical Results} \label{subsec: theory main}
For our analysis, we make some common assumptions on the underlying dynamics $\vf^*$.
\subsection{Assumptions}
We first make continuity assumptions on the system. These assumptions are common in the control theory~\citep{khalil2015nonlinear} and reinforcement learning literature~\citep{curi2020efficient,sussex2022model, sukhija2024optimistic}.
\vspace{0.5em}
\begin{assumption}[Continuous closed-loop dynamics, bounded rewards, and Gaussian noise.]
\label{ass:lipschitz_continuity}
The dynamics model $\vf^*$ and all $\vpi \in \Pi$ are continuous. Furthermore, we assume that the reward is bounded, i.e., $r :\setX \times \setU \to [0, R_{\max}]$,  and process noise is i.i.d.~Gaussian\footnote{For clarity of exposition, we focus on the setting with Gaussian noise. In \cref{appendix: theory}, we also perform the analysis for the more general sub-Gaussian noise case.
} with variance $\sigma^2$, i.e., $\vw_t \stackrel{\mathclap{i.i.d}}{\sim} \setN(\vzero, \sigma^2\mI)$.
\end{assumption}

\ombrl learns an uncertainty-aware model of the true dynamics $\vf^*$, i.e., a mean  $\vmu_n$ and uncertainty $\vsigma_n$ estimate. 
In our theoretical analysis, we focus on Gaussian Process (GP) dynamics models. For GPs, $\vmu_n$ and $\vsigma_n$ have a closed-form solution, and our learned model is calibrated (c.f.,\citet{rothfuss2023hallucinated} or \cref{sec: well calibration section} for the definition of well-calibrated models or \citet{kuleshov2018accurate} on calibration for BNNs). Generally, our guarantees can be extended to broader classes of well-calibrated models, e.g., BNNs (similar to~\citet{curi2020efficient}).

We assume that $\vf^*$ resides in a Reproducing Kernel Hilbert Space (RKHS) of vector-valued functions.

\vspace{0.75em}
\begin{assumption}
We assume that the functions $f^*_j$, $j \in \setmath{1, \ldots, d_\vx}$ lie in a RKHS with kernel $k$ and have a bounded norm $B$, that is $\vf^* \in \setH^{d_\vx}_{k, B}$, with $\setH^{d_\vx}_{k, B} = \{\vf \mid \norm{f_j}_k \leq B, j=1, \dots, d_\vx\}$. Moreover, we assume that $k(\vz, \vz) \leq \sigma_{\max}$ for all $\vx \in \setX$.
\label{ass:rkhs_func}
\end{assumption}
\cref{ass:rkhs_func} ensures that we can model and learn $\vf^*$ with GPs. This assumption is common in the Bayesian optimization~\citep {srinivas, chowdhury2017kernelized} and RL literature~\citep{kakade2020information, curi2020efficient}. Moreover, GPs are nonparametric models and can learn very complex classes of nonlinear functions~\citep{rasmussen2005gp}.

The posterior mean ${\bm \mu}_n(\vz) = [\mu_{n,j} (\vz)]_{j\leq d_\vx}$ and epistemic uncertainty $\vsigma_n(\vz) = [\sigma_{n,j} (\vz)]_{j\leq d_\vx}$ can then be obtained using the following formula
\begin{equation}
\begin{aligned}
\label{eq:GPposteriors}
        \mu_{n,j} (\vz)& = {\bm{k}}_{n}^\top(\vz)({\bm K}_{n} + \sigma^2 \bm{I})^{-1}\vy_{1:n}^j 
        ,  \\
     \sigma^2_{n, j}(\vz) & =  k(\vz, \vz) - {\bm k}^\top_{n}(\vz)({\bm K}_{n}+\sigma^2 \bm{I})^{-1}{\bm k}_{n}(\vz),
\end{aligned}
\end{equation}
Here, $\vy_{1:n}^j$ corresponds to the noisy measurements of $f^*_j$, i.e., the observed next state from the transitions dataset $\setD_{1:n}$,
$\vk_n(\vz) = [k(\vz, \vz_i)]_{\vz_i \in \setD_{1:n}}$, and $\bm{K}_n = [k(\vz_i, \vz_l)]_{\vz_i, \vz_l \in \setD_{1:n}}$ is the data kernel matrix. The restriction on the kernel $k(\vz, \vz) \leq \sigma_{\max}$ implies boundedness of $\vf^*$ and has also appeared in works studying the episodic setting for nonlinear dynamics~\citep{mania2020active, kakade2020information, curi2020efficient, wagenmaker2023optimal, sukhija2024optimistic}. We can also define $\vf^*$ such that $\vx_k = \vx_{k-1} + \vf^*(\vx_{k-1}, \vu_{k-1}) + \vw_{k-1}$ in which case the boundedness of $\vf^*$ captures many real-world systems.

Our theoretical results depend on the {\em maximum information gain} of kernel $k$~\citep{srinivas}, defined as
\begin{equation*}
    {\Gamma}_{N}(k) = \max_{\setA \subset \setX \times \setU; |\setA| \leq N}  \frac{1}{2}\log\det{\mI + \sigma^{-2} {\bm K}_{N}}.
\end{equation*}
$\Gamma_{N}$ is a measure of the complexity for learning $\vf^*$ from $N$ episodes and
is sublinear for many kernels (e.g., $\setO(\log^{d_{x} + d_{u} +1}(N))$ for the squared exponential (RBF) kernel, $\setO((d_{x} + d_{u})\log(N))$ for the linear kernel). In \cref{appendix: theory}, we report the dependence of $\Gamma_N$ on $N$ in \cref{table: gamma magnitude bounds for different kernels}. 


Next, we present the following Lemma, which states that $J_n(\vpi_n)$ from \cref{eq:optimistic plan} is an optimistic estimate of $J(\vpi^*)$.

\vspace{0.3em}
\begin{lemma}
Let \cref{ass:lipschitz_continuity} and \cref{ass:rkhs_func} hold. Then, there exists a $\lambda_n \in \Theta(\sqrt{\Gamma_N})$, such that we have $\forall n > 0$, $\vpi \in \Pi$, with probability at least $1-\delta$, that $J(\vpi) \leq J_n(\vpi)$. Moreover, we have $J(\vpi^*) \leq J_n(\vpi_n)$.
\label{cor: optimism}
\end{lemma}
\textit{Proof Sketch}: Crucially, we leverage the policy difference lemma~\citep{kakade2002approximately},  which bounds the difference between the performance of two different policies with the advantage function. However, we study this lemma for a fixed policy but different dynamics, effectively 
obtaining a simulation lemma~\citep{kearns2002near} for our setting. Next, we bound the difference in performance between the mean dynamics and the true one and show that this is proportional to the model epistemic uncertainty.

\looseness=-1
\cref{cor: optimism} shows that for all policies $\vpi \in \Pi$, $J_n(\vpi)$ gives an upper bound on the true return $J(\vpi)$. This result is of independent interest and can be applied to settings beyond online RL, such as safe RL~\citep{brunke2022safe, as2024actsafe} and offline RL~\citep{levine2020offline, yu2020mopo, rigter2022rambo}. The exact bound for $\lambda_n$ is provided in \cref{lemma: Main Lemma Gaussian noise} in \cref{appendix: theory}.

\looseness=-1
Finally, we present our main theorem, which bounds the regret of \ombrl.
\vspace{0.5em}
\begin{theorem}[Finite horizon setting]
Let \cref{ass:lipschitz_continuity} and \cref{ass:rkhs_func} hold. Then we have $\forall N > 0$ with probability at least $1-\delta$, 
   $R_N \leq \setO\left(\Gamma^{\sfrac{3}{2}}_{N}\sqrt{N}\right)$.
\label{thm: finite horizon regret}
\end{theorem}
\textit{Proof Sketch}: To bound the regret, we first prove that \cref{eq:optimistic plan} is an optimistic estimate of \cref{eq:finite horizon setting}. Then, we use \cref{cor: optimism} to bound the difference in performance for extrinsic rewards between the mean and true dynamics with the epistemic uncertainty of the collected rollout. Next, we analyse the intrinsic reward term and also show that it's bounded by the epistemic uncertainty of the collected rollout. This allows us to relate the cumulative regret $R_N$ with the information gain $\Gamma_N$ and obtain the final bound.


\looseness=-1
\Cref{thm: finite horizon regret} guarantees sublinear regret for a rich class of RKHS functions. Accordingly, for many RKHS, our algorithm enjoys the same asymptotic guarantees as~\citet{kakade2020information}. Note that the regret bound from \citet{kakade2020information} is an order of $\sqrt{\Gamma_{N}}$ better. On the other hand, \ombrl is a much simpler and more scalable algorithm. In \cref{appendix: theory}, we show that \ombrl improves the regret bound from \citet{curi2020efficient}, for the sub-Gaussian noise case, by an exponential factor of $\Gamma^T_N$. Below, we also provide our regret bounds for the $\gamma$-discounted and the non-episodic setting\footnote{Both \citet{kakade2020information} and \citet{curi2020efficient} do not provide a regret bound for these settings.}.
\vspace{0.3em}
\begin{theorem}[$\gamma$-discounted, infinite horizon setting]
Let $R_N = \sum^N_{n=1} J_{\gamma}(\vpi^{*}) - J_{\gamma}(\vpi_n)$.
Under the \cref{ass:lipschitz_continuity} and \cref{ass:rkhs_func}, we have for the $\gamma$-discounted infinite (\cref{eq: discounted}) horizon setting $\forall N > 0$  that
with probability at least $1-\delta$,
   $R_N \le \setO\left(\Gamma^{\sfrac{3}{2}}_{N\log(N)}\sqrt{N}\right)$.
\label{thm: theorem discounted setting}
\end{theorem}
\textit{Proof Sketch}: \looseness=-1
The regret decomposition for this setting is similar to the finite-horizon case (\cref{thm: finite horizon regret}). However,
in contrast to the finite-horizon case, where each episode has a fixed length $T$, we care about the infinite horizon in the $\gamma$-discounted case.
Therefore, to obtain sublinear regret for this setting, we require the agent to observe the system for longer horizons, i.e., the horizon $T(n) \to \infty$ for $n \to \infty$.
We ensure this by picking $T(n) \in \Theta(\log(n))$ and show that this is sufficient to obtain sublinear regret. 

\looseness=-1
In \cref{thm: theorem discounted setting}, we show that even though we truncate each episode after $T(n)$ steps, \ombrl has sublinear regret w.r.t.~the infinite horizon objective. Moreover, the regret for this setting follows the same structure as for the finite horizon case. 
To the best of our knowledge, we are the first to give a regret bound for optimistic model-based RL algorithms for the $\gamma$-discounted setting. 

Finally, we give our regret bound for the non-episodic setting. As pointed out in \citet{kakade2003sample, sharma2021autonomous}, this is the most challenging and closest setting for learning directly in the real-world. \citet{sukhija2024neorl} show that optimistic exploration methods have sublinear regret for the nonepisodic setting. However, their proposed algorithm is intractable in practice.

\vspace{0.5em}
\begin{theorem}[Informal statement; nonepisodic average reward case]
Let $R_N = \sum^N_{n=1} \E[ 
    J_{\text{avg}}(\vpi^*) - r(\vx_n, \vpi_n(\vx_n))
    ]$.
Under the same assumptions as \citet{sukhija2024neorl}, we have for the average reward setting (\cref{eq: average reward}) $\forall N > 0$  that
with probability at least $1-\delta$,
    $R_N \leq \setO\left(\Gamma^{\sfrac{3}{2}}_{N}\sqrt{N}\right)$.
\label{thm: theorem informal average reward setting}
\end{theorem}
\textit{Proof Sketch}: 
The regret analysis for this setting is, in spirit, similar to the finite horizon case. However,
 in the nonepisodic setting,  we cannot reset the agent and have to learn from a single trajectory. Therefore, unlike the episodic case, where we update our model and policy after every episode, in the nonepisodic setting we have to decide when to update the agent. 
 For \ombrl we show that if we only update once we have accumulated enough information, i.e., $\sum^{T(n)-1}_{t=0} \norm{\vsigma_n(\vx_{t, n}, \vpi_n(\vx_{t, n}))} > C$, for a positive constant $C$, then \ombrl has sublinear regret. Intuitively, the model uncertainty measures how much information our agent has acquired since its last update at time step $T(n-1)$. We only update the agent once the information exceeds the threshold $C$.

\looseness=-1
In contrast to \citet{sukhija2024neorl}, \ombrl is much more tractable, and in \cref{thm: theorem informal average reward setting} we show that \ombrl also has sublinear regret in the nonepisodic setting and therefore offers a theoretically strong and practical alternative for model-based exploration for this case. 

\looseness=-1
In this section, we have shown \ombrl, which maximizes a combination of the extrinsic reward and the model epistemic uncertainty, enjoys sublinear regret for common kernels and RL settings. There are principled exploration algorithms designed individually for these settings, e.g., \citep{kakade2020information, curi2020efficient} for the finite-horizon case and \citet{sukhija2024neorl} for the non-episodic case. However, they are often intractable/computationally prohibitive. In contrast, \ombrl works across the different RL settings while also being more practical and scalable.

\looseness=-1
We present additional theoretical results, for example, a sample complexity bound for unsupervised RL algorithms such as~\citet{sekar2020planning, buisson2020actively, sukhija2024optimistic} and a regret bound for the sub-Gaussian noise setting in \cref{appendix: theory}. Our detailed proofs are also provided in \cref{appendix: theory}.

\subsection{Selecting $\lambda_n$ in practice}
\label{sec: selecting lambda}
\looseness=-1
The parameter $\lambda_n$ controls the exploration-exploitation trade-off for \ombrl. 
In \cref{appendix: theory} we provide the theoretical bound for $\lambda_n$, however in practice, $\lambda_n$ is treated as a hyperparameter. This is similar to other optimistic exploration and intrinsic exploration algorithms~\citep{burda2018exploration, kakade2020information, curi2020efficient}, which also heuristically select the amount of exploration. \citet{sukhija2024maxinforl} 
empirically study combining extrinsic and intrinsic rewards for model-free algorithms and propose an approach for automatically tuning the intrinsic reward coefficient, i.e., $\lambda_n$. 
We find their approach works well for our state-based and visual control tasks. Moreover, we describe their approach and how we choose $\lambda_n$ for our experiments in~\cref{appendix: experiment_details}.

\begin{figure}[t]
  \centering
  \begin{overpic}[width=\textwidth]{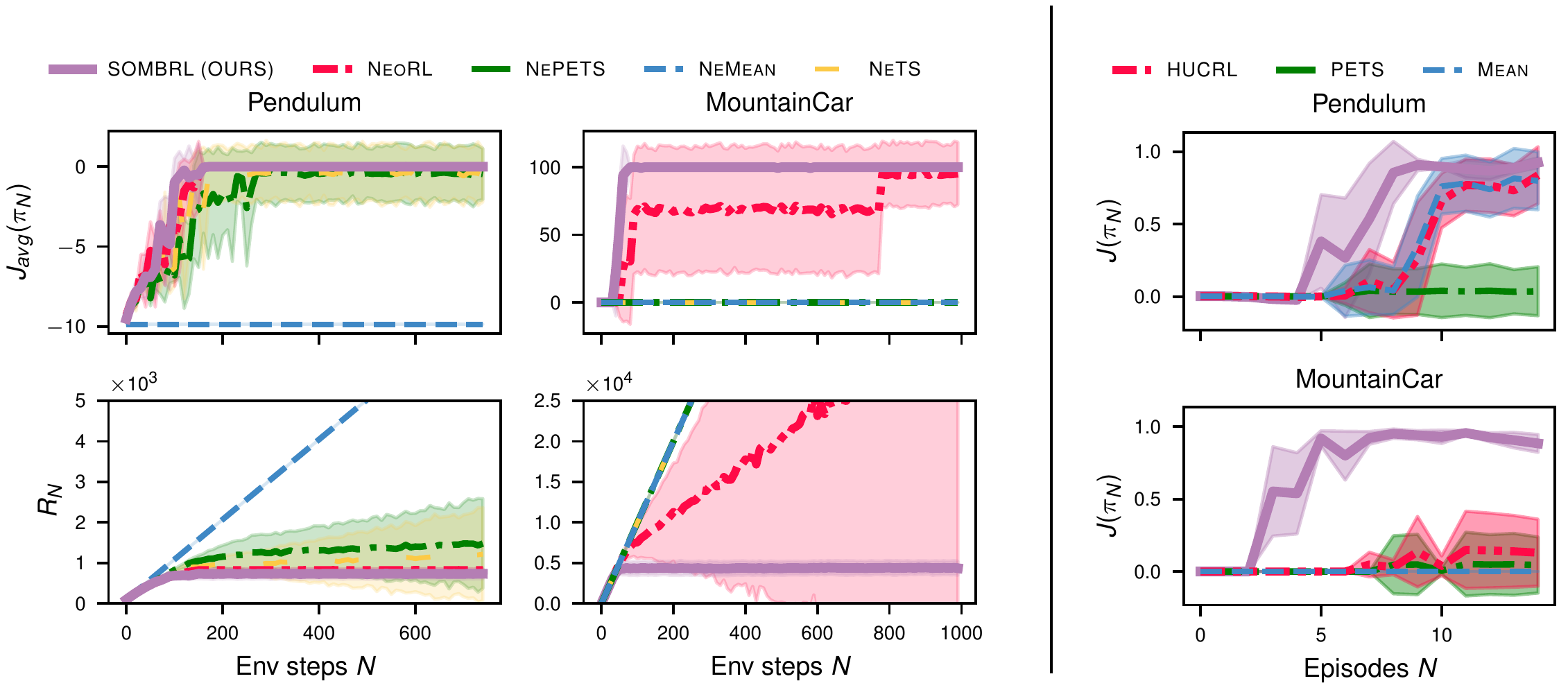}
    \put(23,41.5){Nonepisodic Setting}
    \put(78,41.5){Episodic Setting}
  \end{overpic}
  \caption{
  \emph{Left:} Learning curves for the nonepisodic setting with GP dynamics. We report the average reward $J_{avg}(\vpi_N)$ and regret $R_N$. The curves are reported with 5 seeds, and we plot the median return with its standard deviation.
  \emph{Right:} Learning curves for the episodic setting with GP dynamics. We report the median episode reward $J(\vpi_N)$ over an episode with 5 seeds and its standard deviation.}
  \label{fig:overlay}
\end{figure}

\subsection{Application of \ombrl with GP dynamics}
Finally, we empirically validate our theoretical findings for the GP case in \cref{fig:overlay}, where we
compare \ombrl to HUCRL~\citep{curi2020efficient}, PETS~\citep{chua2018pets}, and greedy (mean) planning in the episodic setting. For the nonepisodic setting, we consider their nonepisodic counterparts as proposed in \citet{sukhija2024neorl}. We evaluate the algorithms on the Pendulum and MountainCar tasks from the OpenAI Gym benchmark~\citep{brockman2016openai}. From the experiments, we conclude that \ombrl performs the best across all baselines for both the episodic and the nonepisodic setting. Moreover, while HUCRL and \neorl, which explore according to \cref{eq:optimistic plan expensive}, perform better than other baselines, they are worse than \ombrl. We believe this is because of the practical challenges associated with solving the optimization problem in \cref{eq:optimistic plan expensive}. Moreover, solving \cref{eq:optimistic plan expensive} is also computationally more expensive. For instance, in our experiments, HUCRL requires roughly $3\times$ more compute time than OMBRL (see \cref{subsec:compute cost}).

\section{Experiments} \label{sec: experiments}
In our experiments, we showcase the flexibility and scalability of \ombrl by
combining it with three different model-based RL algorithms; (\emph{i}) MBPO~\citep{janner2019trust} for state-based tasks, \textsc{Dreamer}~\citep{hafner2023mastering} for visual control tasks, and \textsc{SimFSVGD}~\citep{rothfuss2024bridging} for our hardware experiment on the RC car. 
Note that principled exploration methods such as \citet{kakade2020information, curi2020efficient} do not scale to the settings, such as visual control tasks, considered in this work.
We consider the DeepMind control (DMC) benchmark~\citep{tassa2018deepmind} for the state-based and visual control tasks and test on environments with varying dimensionality\footnote{including the humanoid from DMC: $d_\vx = 67$, $d_\vu = 21$}. We also evaluate on several environments from the Atari benchmark~\citep{bellemare2013arcade} for the visual control tasks. In all our experiments, we report the episodic returns using the median over 5 seeds along with its standard deviation. We provide additional experiment details in \cref{appendix: experiment_details}.

\paragraph{State-based experiments}
\looseness=-1
We refer to the MBPO version of \ombrl as \textsc{MBPO-Optimistic}. The resulting algorithm operates similarly to~\citet{janner2019trust} and trains a policy from real and model-generated rollouts to maximize the extrinsic and intrinsic rewards. For the policy training, we use the soft actor-critic (SAC) algorithm~\citep{haarnoja2018soft},
and for the intrinsic reward coefficient, $\lambda_n$, we use the auto-tuning approach from~\citet{sukhija2024maxinforl}. 
We train an ensemble of dynamics models and use their disagreement to quantify the epistemic uncertainty. 
As baselines, we consider (\emph{i}) \textsc{MBPO-Mean}, which maximizes only the extrinsic reward, i.e., $\lambda_n = 0$, and (\emph{ii}) \textsc{MBPO-PETS}, which is based on the PETS algorithm~\citep{chua2018pets} maximizing the extrinsic rewards in expectation over the ensemble dynamics (see~\cref{eq: greedy sampling}).
We report the results on the left side of~\cref{fig:panel1}. We conclude that across all tasks, \textsc{MBPO-Optimistic} performs the best. Particularly, in sparse reward tasks such as the Mountaincar and CartPole, \textsc{MBPO-Optimistic} successfully solves the task whereas the greedy baselines fail. \textsc{MBPO-Optimistic} also successfully scales to high dimensional problems such as the Quadruped and Humanoid environments. We provide additional experiments with \textsc{MBPO-Optimistic} in \cref{appendix: additional exps}, where we evaluate it on more environments and compare it with pure off-policy algorithms SAC and MaxInfoRL~\citep{sukhija2024maxinforl}.

\paragraph{Visual control experiments}
We investigate the scalability of \ombrl to challenging and high-dimensional problems by evaluating it on visual control tasks. We combine \ombrl with \textsc{Dreamer}~\citep{hafner2023mastering}, an MBRL algorithm for visual control problems,
and call the resulting algorithm \textsc{Dreamer-Optimistic}. We use the same approach as \citet{sekar2020planning} for quantifying the epistemic uncertainty and for selecting the intrinsic reward coefficient, $\lambda_n$, we use the auto-tuning approach from~\citet{sukhija2024maxinforl}. 
We report the results on the right side of~\cref{fig:panel1}. Overall, \textsc{Dreamer-Optimistic} 
performs on-par with \textsc{Dreamer} on most tasks and
outperforms it on the Finger-spin task from DMC and the Venture task from the Atari benchmark. Particularly, for Venture, a sparse reward task, Dreamer fails to achieve any reward. 
\begin{figure*}[t]
    \centering
      \begin{overpic}[width=\textwidth]{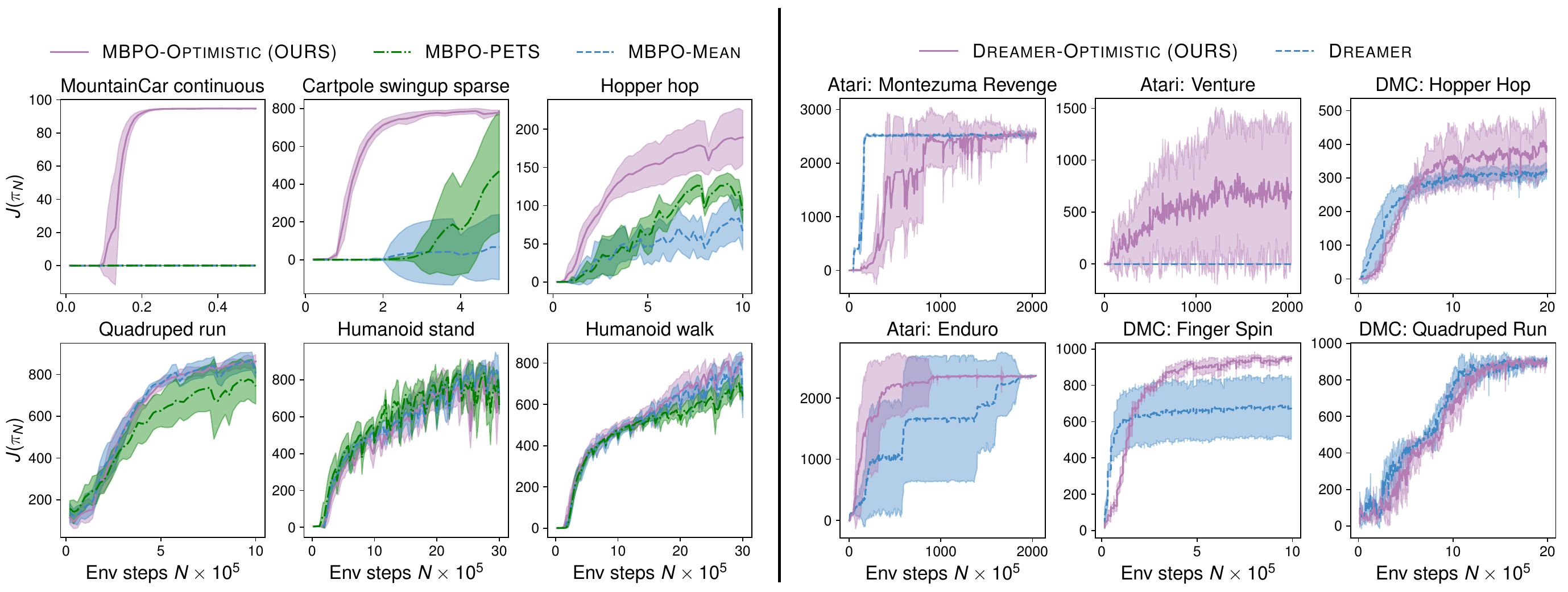}
    \put(12,36){\small Model based policy optimization}
    \put(65,36){\small Visual control tasks}
  \end{overpic}
    \caption{\emph{Left:} Learning curves for the state-based tasks from DMC using MBPO as the base algorithm. Across all experiments, \textsc{MBPO-Optimistic} obtains the best performance compared to its greedy variants. \textsc{MBPO-Optimistic} also scales to high-dimensional tasks, specifically the humanoid environments from DMC. 
    \emph{Right:} 
    Learning curves for the visual control tasks from DMC and Atari using \textsc{Dreamer} as the base algorithm. \textsc{Dreamer-Optimistic} either performs on-par or better than \textsc{Dreamer} in all our experiments. Particularly, in the Venture task from the Atari benchmark, where \textsc{Dreamer} fails to obtain any rewards.}
            \label{fig:panel1}
\end{figure*}
\begin{figure*}[th]
    \centering
          \begin{overpic}[width=\textwidth]{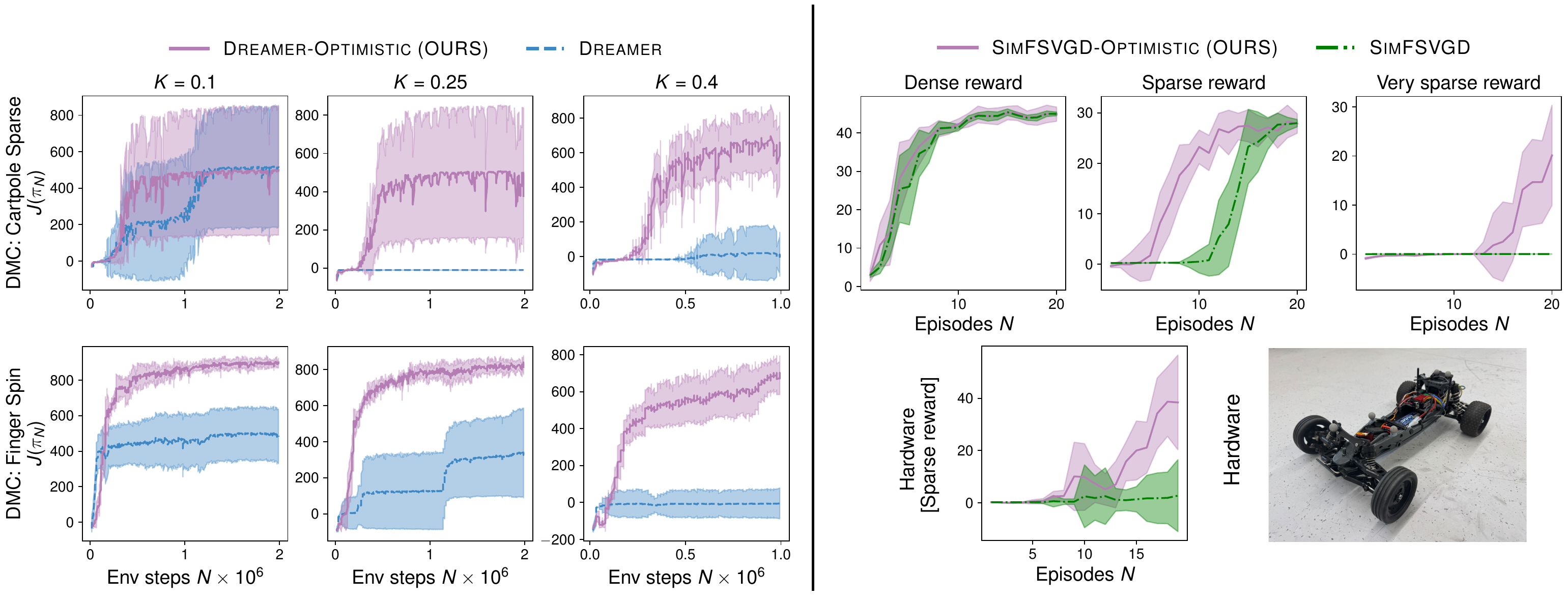}
    \put(10,37){\small Visual control learning with action cost}
    \put(67,37){\small Online RL on RC Car}
  \end{overpic}
    \caption{\emph{Left:} Learning curves with action costs, where we compare \textsc{Dreamer} with \textsc{Dreamer-Optimistic}. \textsc{Dreamer} fails to explore sufficiently with  action costs, whereas \textsc{Dreamer-Optimistic} is able to explore and obtain much higher performance. 
    \emph{Right:}
    Learning curves for our experiments with \textsc{SimFSVGD}. \emph{Top row}: We change the parameters of the reward function from \citet{rothfuss2024bridging}, and make it sparse, starting from their dense reward. We observe that, as the reward gets sparser, \textsc{SimFSVGD} drops in performance and \textsc{SimFSVGD-Optimistic} outperforms it. \emph{Bottom row}: We run the sparse reward configuration on hardware (depicted on the right side at the bottom), where we obtain similar results. As opposed to \textsc{SimFSVGD-Optimistic}, \textsc{SimFSVGD} fails to solve the task.}
            \label{fig:panel2}
\end{figure*}
\looseness=-1
\citet{curi2020efficient} study the sensitivity of greedy exploration algorithms w.r.t.~the action penalties in the reward. Inspired by their experiments, we modify the reward for the CartPole and Finger spin environments by adding an action cost, $r_{\text{action}}(\va) = -K \norm{\va}_2$, where $K$ controls the penalty for large actions. \citet{curi2020efficient} show that even for small action costs, greedy exploration methods fail, converging to the sub-optimal solution of applying small actions. We observe a similar outcome in \cref{fig:panel2} (left side), where \textsc{Dreamer} fails to solve the tasks for both the Finger spin and CartPole environments. On the other hand, \textsc{Dreamer-Optimistic} achieves much higher returns due to its optimistic exploration. 

We provide additional experiments with \textsc{Dreamer-Optimistic}, including more environments and proprioceptive tasks in \cref{appendix: additional exps}. 

\paragraph{Hardware experiments}    \looseness=-1 \citet{rothfuss2024bridging} propose a novel approach for training deep Bayesian models that incorporates low-fidelity physical priors. Their approach significantly improves sample efficiency, which they illustrate in their hardware experiments on an RC car. Inspired by their experimental setup, we conduct a similar experiment on a highly dynamic RC car. The task is to perform a complex parking maneuver with drifting as depicted in \cref{fig:rccar_teaser}. 
We use the same experimental setup as \citet{rothfuss2024bridging}.

First, we evaluate our algorithm \textsc{SimFSVGD-Optimistic}, a combination of \ombrl and \textsc{SimFSVGD}, in simulation. The simulation is based on a realistic race car simulation from~\citet{kabzan2020amz}. For the simulation experiments, we ablate different choices for the reward parameters, starting with the dense reward configuration from \citet{rothfuss2024bridging} and adapt parameters to obtain sparser rewards (see \cref{appendix: experiment_details} for more detail). We report the results in the top row of~\cref{fig:panel2}. For the dense reward setting, \textsc{SimFSVGD} and \textsc{SimFSVGD-Optimistic} perform similarly, but for sparser rewards, \textsc{SimFSVGD-Optimistic} outperforms \textsc{SimFSVGD}. In particular, for the setting with very sparse rewards, \textsc{SimFSVGD}  completely fails to solve the task. We conduct our hardware experiments using the sparse reward configuration, and report the learning curve in the bottom row of~\cref{fig:panel2}. In line with the simulation experiments, we also observe similar behavior on hardware. \textsc{SimFSVGD-Optimistic} learns to solve the task, whereas \textsc{SimFSVGD} completely fails. In fact, out of 5 attempts, \textsc{SimFSVGD} worked only once, and otherwise converged to a local optimum of not moving from the starting position\footnote{Video is available at \url{https://sukhijab.github.io/projects/sombrl/}}.




\section{Conclusion}
In this work, we propose \ombrl, which maximizes a weighted sum of the extrinsic reward and the agent's epistemic uncertainty.
We show that \ombrl effectively performs optimistic exploration and provide regret bounds for it in a variety of settings, in particular for continuous state-action spaces and many common classes of RL problems, namely, finite-horizon, infinite-horizon, episodic, and non-episodic RL. To the best of our knowledge, we are the first to provide these theoretical guarantees, yielding a more flexible, scalable, and principled algorithm for exploration.
We illustrate the strengths of \ombrl in our experiments, where we combine \ombrl with different model-based RL algorithms, evaluate it on tasks of varying dimensionality, including visual control problems, and also illustrate the benefits of optimistic exploration on hardware. In all cases, \ombrl achieves the best performance, being significantly more scalable and computationally cheaper than prior optimistic exploration methods and stronger at exploration than SOTA deep RL baselines. 

\looseness=-1
A limitation of \ombrl is that it requires training a probabilistic model, e.g., deep ensembles, for quantifying epistemic uncertainty. We report the computational cost of our method in \cref{subsec:compute cost} and show that the computational overhead is small compared to the training cost for the actor and critic. 


\subsubsection*{Acknowledgments}
We thank Scott Sussex for the insightful discussion and feedback on this work. 
This project has received funding from the Swiss National Science Foundation under NCCR Automation, grant agreement 51NF40 180545, the Microsoft Swiss Joint Research Center, and the SNSF Postdoc Mobility Fellowship 211086. Bhavya Sukhija was gratefully supported by ELSA (European Lighthouse on Secure and Safe AI) funded by the European Union under grant agreement No. 101070617. 
\bibliography{references}
\bibliographystyle{icml2025}

\clearpage

\appendix
\section{Related Work}
\label{sec: related works}
\paragraph{Deep model-based RL}
Model-based RL algorithms offer a sample-efficient solution for learning directly in the real world~\cite{hansen2022modem, wu2023daydreamer, rothfuss2024bridging}. Most widely applied algorithms~\citep{chua2018pets,
janner2019trust,
hafner2023mastering, hansen2023td} commonly rely on naive exploration techniques such as Boltzmann exploration and differ primarily in the type of dynamics modeling and policy planners. 
\citet{cesa2017boltzmann} show that Boltzmann exploration is suboptimal even in the simplified setting of stochastic bandits. \ombrl is agnostic to the choice of modeling and planners, as we demonstrate in \cref{subsec: theory main} and \ref{sec: experiments}. Moreover, we focus on the problem of exploration for MBRL and propose a principled exploration approach. We derive regret bounds for 
our approach, showing that it is theoretically grounded. Furthermore, we illustrate the benefits of principled exploration in our hardware experiment, where the naive exploration baseline fails to obtain any meaningful exploration. To the best of our knowledge, we are the first to propose a simple, flexible, scalable, and theoretically grounded approach for principled exploration and show its benefits directly in the real world.

\vspace{-0.5em}
\paragraph{Theoretical results for Model-based RL} 
There are numerous works that study MBRL for linear dynamical systems theoretically~\citep{abbasi2011regret, cohen2019learning, simchowitz2020naive, dean2020sample, faradonbeh2020optimism, abeille2020efficient,
treven2021learning}, focusing primarily on the challenges of nonepisodic learning. In the nonlinear case, \citet{kakade2020information, 
curi2020efficient, mania2020active,wagenmaker2023optimal,
treven2024ocorl} analyze the finite-horizon episodic setting and provide regret bounds that are sublinear for many RKHS. Recently, \citet{sukhija2024neorl} extended these results to the nonepisodic setting. Crucially, most of these algorithms are based on the principle of optimism in the face of uncertainty and require solving the problem in \cref{eq:optimistic plan expensive}. As highlighted in \cref{sec: intro} and \ref{sec: exploration strategies in MBRL}, solving these problems is often intractable or computationally expensive. Therefore, naive exploration techniques, such as Boltzmann exploration are more widely used. \ombrl addresses this drawback and proposes an alternative optimistic exploration method, which is much simpler and more scalable. Furthermore, it enjoys the same asymptotic guarantees as these methods and hence is also theoretically grounded.

\vspace{-0.5em}
\paragraph{Intrinsic exploration in RL}
\looseness=-1
Intrinsic rewards are often used as a surrogate objective for principled exploration in challenging tasks \citep[see][for a comprehensive survey]{aubret2019survey}. 
 Common choices of intrinsic rewards are model prediction error or ``Curiosity''~\citep{schmidhuber1991possibility, pathak2017curiosity, burda2018exploration}, novelty of transitions/state-visitation counts~\citep{stadie2015incentivizing,bellemare2016unifying}, diversity of skills/goals~\citep{eysenbach2018diversity, DADS, nair2018visual, skewfit}, empowerment~\citep{klyubin2005empowerment, salge2014empowerment},
    and information gain of the dynamics~\citep{sekar2020planning, mendonca2021discovering, sukhija2024optimistic}. 
    However, these rewards are mostly used for pure exploration and rarely considered in combination with the extrinsic reward. We show that combining the model epistemic uncertainty, an intrinsic reward, with the extrinsic one, effectively performs optimistic exploration, thus, providing a theoretical grounding for our approach.
    There are a few works from bandits~\citep{auer2002finite, srinivas}, data-driven control~\citep{aastrom1971problems, 
chiuso2023harnessing, grimaldi2024bayesian}, and RL~\citep{abeille2020efficient, sukhija2024maxinforl} that have also proposed maximizing extrinsic rewards jointly with epistemic uncertainty. 
The data-driven control community refers to this as the separation principle between model identification and control design~\citep{aastrom1971problems, chiuso2023harnessing, grimaldi2024bayesian}. In RL, \citet{abeille2020efficient} show duality between \cref{eq:optimistic plan expensive} and \cref{eq:optimistic plan} for linear systems. For nonlinear systems and deep RL,
\citet{sukhija2024maxinforl} empirically study combining extrinsic and intrinsic rewards. 
However, compared to these works, we 
demonstrate \ombrl's scalability to high-dimensional settings such as visual control tasks, and
additionally, ground this approach theoretically, providing regret bounds for nonlinear systems and common RL settings. 
\clearpage
\section{Proofs} \label{appendix: theory}
\subsection{Definition of Well-calibrated Model} \label{sec: well calibration section}
 We define a well-calibrated statistical model of $\vf^*$, which captures both the mean prediction $\vmu_n$ and the uncertainty $\vsigma_n$ of our learned model.
\begin{definition}[Well-calibrated statistical model of $\vf^*$, \cite{rothfuss2023hallucinated}]
\label{definition: well-calibrated model}
    Let $\setZ \defeq \setX \times \setU$.
    A sequence of sets $\{\setM_{n}(\delta)\}_{n \ge 0}$, where
    \begin{align*}
        \setM_n(\delta) \defeq &\left\{\vf: \setZ \to \R^{d_\vx} \mid \forall \vz \in \setZ, \forall j \in \setmath{1, \ldots, d_\vx}: \abs{\mu_{n, j}(\vz) - f_j(\vz)} \le \beta_n(\delta) \sigma_{n, j}(\vz)\right\},
    \end{align*}
    is an all-time well-calibrated statistical model of the function $\vf^*$,
    if, with probability at least $1-\delta$, we have $\vf^* \in \bigcap_{n \ge 0}\setM_n(\delta)$.
    Here, $f_{j}$, $\mu_{n, j}$ and $\sigma_{n, j}$ denote the $j$-th element in the vector-valued functions $\vf$, $\vmu_n$ and $\vsigma_n$ respectively, and $\beta_n(\delta) \in \Rzero$ is a scalar function that depends on the confidence level $\delta \in (0, 1]$ and which is monotonically increasing in $n$. 
\end{definition}

In the following, we show that Gaussian processes are well-callibrated models if \cref{ass:rkhs_func} holds. 
\begin{lemma}[Well calibrated confidence intervals for RKHS, \citet{rothfuss2023hallucinated}]
    Let $\vf^* \in \setH_{k,B}^{d_\vx}$.
Suppose ${\vmu}_n$ and $\vsigma_n$ are the posterior mean and variance of a GP with kernel $k$, \Cref{eq:GPposteriors}.
There exists $\beta_n(\delta)$, for which the tuple $(\vmu_n, \vsigma_n, \beta_n(\delta))$ is a well-calibrated statistical model of $\vf^*$.
\label{lem:rkhs_confidence_interval}
\end{lemma}
\looseness=-1
In summary, in the RKHS setting, a GP is a well-calibrated model. 
\subsection{Analysis for the finite horizon case}
\begin{lemma}
Let \cref{ass:lipschitz_continuity} and \cref{ass:rkhs_func} hold. 
Consider the following definitions
\begin{align*}
    &J(\vpi, \vf^*) = \E_{\vf^*}\left[\sum_{t=0}^{T-1} r(\vx_t, \vpi(\vx_{t}))\right], \; \text{s.t., }  \vx_{t+1} = \vf^*(\vx_{t}, \vpi(\vx_{t})) + \vw_t, \quad \vx_0 = \vx(0). \\
    &J(\vpi, \vmu_n) = \E_{\vmu_n}\left[\sum_{t=0}^{T-1} r(\vx'_t, \vpi(\vx'_{t}))\right], \; \text{s.t., } \vx'_{t+1} = \vmu_n(\vx'_{t}, \vpi(\vx'_{t})) + \vw_t, \quad \vx'_0 = \vx(0). \\
    &\Sigma_n(\vpi, \vf^*) = \E_{\vf^*}\left[\sum_{t=0}^{T-1} \norm{\vsigma_n(\vx_t, \vpi(\vx_{t}))}\right], \; \text{s.t., }  \vx_{t+1} = \vf^*(\vx_{t}, \vpi(\vx_{t})) + \vw_t, \quad \vx_0 = \vx(0). \\
    &\Sigma_n(\vpi, \vmu_n) = \E_{\vmu_n}\left[\sum_{t=0}^{T-1} \norm{\vsigma_n(\vx'_t, \vpi(\vx'_{t}))}\right]. \; \text{s.t., } \vx'_{t+1} = \vmu_n(\vx'_{t}, \vpi(\vx'_{t})) + \vw_t, \quad \vx'_0 = \vx(0). \\
    &\lambda_n = C_{\max} T \frac{(1 + \sqrt{d_x}) \beta_{n-1}(\delta)}{\sigma},
    \end{align*}
    where $C_{\max} = \max\{R_{\max}, \sigma_{\max}\}$.
Then we have for all $n \geq 0$, $\vpi \in \Pi$ with probability at least $1-\delta$
\begin{align*}
     |J(\vpi, \vf^*) - J(\vpi, \vmu_n)| &\leq \lambda_n \Sigma_n(\vpi, \vmu_n) \\
    |J(\vpi, \vf^*) - J(\vpi, \vmu_n)| &\leq \lambda_n \Sigma_n(\vpi, \vf^*) 
\end{align*}
\label{lemma: Main Lemma Gaussian noise}
\end{lemma}

\begin{proof}
    We give the proof for $|J(\vpi, \vf^*) - J(\vpi, \vmu_n)| \leq \lambda_n(L_r, \vmu_n) \Sigma_n(\vpi, \vmu_n)$. The same argument holds for the second inequality.  
    Let $J_{t+1}(\vpi, \vf^*, \vx)$ denote the cost-to-go from state $\vx$, step $t+1$ onwards under the dynamics $\vf^*$.
    Following the Policy difference Lemma from~\citep{kakade2002approximately} and~\citet[Corollary 2.]{sukhija2024optimistic}
    \begin{align*}
        &J(\vpi, \vmu_n) - J(\vpi, \vf^*) =  \E_{\vmu_n}\left[\sum^{T-1}_{t=0} J_{t+1}(\vpi, \vf^*, \vx'_{t+1}) -  J_{t+1}(\vpi, \vf^*, \hat{\vx}_{t+1})\right], \\
  &\text{with } \hat{\vx}_{t+1} = \vf^*(\vx'_{t}, \vpi(\vx'_t)) + \vw_t, \; \text{and } \vx'_{t+1} = \vmu_n(\vx'_{t}, \vpi(\vx'_t))+ \vw_t.
    \end{align*}
    Therefore,
    \begin{align*}
        &|J(\vpi, \vmu_n) - J(\vpi, \vf^*)| =  \left|\E\left[\sum^{T-1}_{t=0} J_{t+1}(\vpi, \vf^*, \vx'_{t+1}) -  J_{t+1}(\vpi, \vf^*, \hat{\vx}_{t+1})\right]\right| \\
 &\leq \sum^{T-1}_{t=0} \E\left[\left|\E_{\vw_t}\left[J_{t+1}(\vpi, \vf^*, \vx'_{t+1}) -  J_{t+1}(\vpi, \vf^*, \hat{\vx}_{t+1})\right]\right|\right]
    \end{align*}
Next, we bound the last term using the derivation from \citet{kakade2020information}. Let $C(\vx) = J^2_{t+1}(\vpi, \vf^*, \vx)$.
\begin{align*}
    &\left|\E_{\vw_t}\left[J_{t+1}(\vpi, \vf^*, \vx'_{t+1}) -  J_{t+1}(\vpi, \vf^*, \hat{\vx}_{t+1})\right]\right| \\
    &\leq \sqrt{\max\left\{\E_{\vw_t}[C(\vx'_{t+1})],  \E_{\vw_t}[C(\hat{\vx}_{t+1})]\right\}} \\
    &\times \min\left\{\frac{\norm{\vf^*(\vx'_t, \vpi(\vx'_t)) - \vmu_n(\vx'_t, \vpi(\vx'_t))}}{\sigma}, 1\right\} \tag*{ \citep[Lemma C.2.]{kakade2020information}}\\
    &\le R_{\max} T \frac{(1 + \sqrt{d_x}) \beta_{n-1}(\delta)}{\sigma} \norm{\vsigma_{n-1}(\vx'_t, \vpi(\vx'_t))}
\end{align*}

Therefore, we have
\begin{align*}
    &|J(\vpi, \vmu_n) - J(\vpi, \vf^*)| \\
 &\leq \sum^{T-1}_{t=0} \E\left[\left|\E_{\vw_t}\left[J_{t+1}(\vpi, \vf^*, \vx'_{t+1}) -  J_{t+1}(\vpi, \vf^*, \hat{\vx}_{t+1})\right]\right|\right] \\
 &\leq \lambda_n \sum^{T-1}_{t=0} \E\left[\norm{\vsigma_{n-1}(\vx'_t, \vpi(\vx'_t))}\right].
\end{align*}
\end{proof}

Note that \cref{cor: optimism} follows directly from \cref{lemma: Main Lemma Gaussian noise}.

\begin{lemma}
    Let \cref{ass:lipschitz_continuity} and \cref{ass:rkhs_func} hold and consider the simple regret at episode $n$, $r_n = J(\vpi^*, \vf^*) - J(\vpi_n, \vf^*)$. The following holds for all $n > 0$ with probability at least $1-\delta$
    \begin{equation*}
        r_n \leq  (2\lambda_n + \lambda^2_n)\Sigma_n(\vpi_n, \vf^*)
    \end{equation*}
    \label{lemma: simple regret}
\end{lemma}
\begin{proof}
    \begin{align*}
        r_n &= J(\vpi^*, \vf^*) - J(\vpi_n, \vf^*) \\
        &\le J(\vpi^*, \vmu_n) + \lambda_n \Sigma_n(\vpi^*, \vmu_n) - J(\vpi_n, \vf^*) \tag{\cref{lemma: Main Lemma Gaussian noise}} \\
        &\leq J(\vpi_n, \vmu_n) + \lambda_n \Sigma_n(\vpi_n, \vmu_n) - J(\vpi_n, \vf^*) \tag{\cref{eq:optimistic plan}} \\
        &= J(\vpi_n, \vmu_n)  - J(\vpi_n, \vf^*)  + \lambda_n \Sigma_n(\vpi_n, \vmu_n) \\
        &\leq \lambda_n \Sigma_n(\vpi_n, \vf^*) + \lambda_n \Sigma_n(\vpi_n, \vmu_n) \tag{\cref{lemma: Main Lemma Gaussian noise}} \\
        &= 2\lambda_n \Sigma_n(\vpi_n, \vf^*) + \lambda_n (\Sigma_n(\vpi_n, \vmu_n) - \Sigma_n(\vpi_n, \vf^*)) \\
        &\leq (\lambda^2_n + 2\lambda_n) \Sigma_n(\vpi_n, \vf^*).
    \end{align*}
    Here in the last inequality, we used the fact that $\norm{\vsigma(\cdot, \cdot)}$ is bounded and positive, therefore, we can treat it similar to the reward (it is in fact an intrinsic reward) and use \cref{lemma: Main Lemma Gaussian noise}.
\end{proof}

\begin{proof}[Proof of \cref{thm: finite horizon regret}]
    \begin{align*}
        R_N &= \sum^N_{n=1} r_n \\
        &\le \sum^N_{n=1}(\lambda^2_n + 2\lambda_n) \Sigma_n(\vpi_n, \vf^*) \\
        &\leq (\lambda^2_N + \lambda_N) \sum^N_{n=1}\Sigma_n(\vpi_n, \vf^*) \\
        &= (\lambda^2_N + 2\lambda_N) \sum^N_{n=1}\E_{\vf^*}\left[\sum_{t=0}^{T-1} \norm{\vsigma_n(\vx_t, \vpi(\vx_{t}))}\right] \\
        &\leq (\lambda^2_N + 2\lambda_N) \sqrt{NT}\sum^N_{n=1}\E_{\vf^*}\left[\sum_{t=0}^{T-1} \norm{\vsigma^2_n(\vx_t, \vpi(\vx_{t}))}\right] \\
        &\leq C(\lambda^2_N + 2\lambda_N)T \sqrt{N\Gamma_{NT}} \tag{\citet[Lemma 17]{curi2020efficient}}
    \end{align*}
    Finally, note that from \cref{lemma: Main Lemma Gaussian noise} we have $\lambda_N \propto T\beta_n$ and $\beta_n \propto \sqrt{\Gamma_n}$~\citep{chowdhury2017kernelized}. Therefore, $R_N \leq \setO(T^3\Gamma^{\sfrac{3}{2}}_N \sqrt{N})$ 
\end{proof}
 \begin{table*}[th]
\begin{center}
\caption{Maximum information gain bounds for common choice of kernels.}
\label{table: gamma magnitude bounds for different kernels}
\begin{tabular}{@{}lll@{}}
\toprule
Kernel&$k(\vx, \vx')$ & $\Gamma_N$ \\ \midrule
    Linear &$\vx^\top \vx'$   & $\mathcal{O}\left(d \log(N)\right)$                    \\
    RBF &$e^{-\frac{\norm{\vx - \vx'}^2}{2l^2}}$& $\mathcal{O}\left( \log^{d+1}(N)\right)$                    \\
    Matèrn &$\frac{1}{\Gamma(\nu)2^{\nu - 1}}\left(\frac{\sqrt{2\nu}\norm{\vx-\vx'}}{l}\right)^{\nu}B_{\nu}\left(\frac{\sqrt{2\nu}\norm{\vx-\vx'}}{l}\right)$  & $\mathcal{O}\left(N^{\frac{d}{2\nu + d}}\log^{\frac{2\nu}{2\nu+d}}(N)\right)$                     \\ \bottomrule
\end{tabular}
\end{center}
\end{table*}
In \cref{table: gamma magnitude bounds for different kernels} we list rates of $\Gamma_N$ for the most common choice of kernels. 

\subsection{Analysis for the discounted infinite horizon case}
For the infinite horizon case, we first study the posterior variance $\vsigma_n$ in the feature space. Moreover, let $\vz = (\vx, \vu)$ and $\setZ = \setX \times \setU$. 

For the ease of notation we denote $\vz_{k, n} = (\vx^{n}_k, \vpi_n(\vx^n_k)$). For $\vz$ we define the kernel embedding $k_{\vz} = k(\vz, \cdot)$. The covariance matrix $\mV_t: \setH \to \setH$ in the feature form is:
\begin{align}
    \mV_{t} = \mI + \frac{1}{\sigma^2}\sum_{i=1}^tk_{\vz_i}k_{\vz_i}^\top.
\end{align}
Note that we have $\vx_{t+1} = \inner{k_{\vz_t}}{\vf^*}_\setH + \vw_t$. With the design matrix $\mM_t: \setH \to \R^{t}$
\begin{align}
    \mM_t = 
    \begin{pmatrix}
     k_{\vz_1} & k_{\vz_2} & \cdots & k_{\vz_t}   
    \end{pmatrix}
\end{align}
we have $\mV_t = \mI + \frac{1}{\sigma^2}\mM_t\mM_t^\top$ and since $\mK_t = \mM_t^\top\mM_t$ we have
\begin{align}
    \determinant(\mV_t) = \determinant\left(\mI + \frac{1}{\sigma^2}\mK_{t}\right)
\end{align}

\begin{corollary}[Lower bound on the posterior log determinant]
\begin{align}
\log\left(\det{\mV_{n}}\right) &\geq \log\left(\det{\mV_{n-1}}\right) + \log\left(1 + \sigma^{-2}\sum^{\widehat{T}_n}_{k=1} \norm{\vsigma_{n-1}(\vz_{k, n})}^2\right)
\end{align}
In particular, we have
\begin{equation}
   \log\left(\frac{\det{\mV_{N}}}{\det{\mV_{0}}}\right) \geq \sum^N_{n=1}\log\left(1 + \sigma^{-2}\sum^{\widehat{T}_n}_{k=1} \norm{\vsigma_{n-1}(\vz_{k, n})}^2\right)
\end{equation}
\label{cor: recursive info gain bound}
\end{corollary}
\begin{proof}
    \begin{align*}
\log\left(\det{\mV_{n}}\right) &=\log\left(\det{\mV_{n-1}}\right) + \log\left(\det{\mI + \sigma^{-2}\mV^{-\sfrac{1}{2}}_{n-1}\sum^{\widehat{T}_n}_{k=1} \vk_{\vz_{k, n}} \vk^{\top}_{\vz_{k, n}} \mV^{-\sfrac{1}{2}}_{n-1}}\right) \\
        &\geq \log\left(\det{\mV_{n-1}}\right) \\ &+ \log\left(1 + \tr\left(\sigma^{-2}\mV^{-\sfrac{1}{2}}_{n-1}\sum^{\widehat{T}_n}_{k=1} \vk_{\vz_{k, n}} \vk^{\top}_{\vz_{k, n}} \mV^{-\sfrac{1}{2}}_{n-1}\right)\right) \tag{see (*) below} \\
        &= \log\left(\det{\mV_{n-1}}\right) +  \log\left(1 + \sigma^{-2}\sum^{\widehat{T}_n}_{k=1} \norm{\vk_{\vz_{k, n}}}_{\mV^{-1}_{n-1}}^2\right) \\
        &= \log\left(\det{\mV_{n-1}}\right) +  \log\left(1 + \sigma^{-2}\sum^{\widehat{T}_n}_{k=1} \norm{\vsigma_{n-1}(\vz_{k, n})}^2\right)
    \end{align*}

    We prove (*) in the following, 
    first let $\vm_{k} = \sigma^{-1}\mV^{-\sfrac{1}{2}}_{n-1}\vk_{\vz_{k, n}}$, then we have 
    \begin{align*}
        &\log\left(\det{\mI + \sigma^{-2}\mV^{-\sfrac{1}{2}}_{n-1}\sum^{\widehat{T}_n}_{k=1} \vk_{\vz_{k, n}} \vk^{\top}_{\vz_{k, n}} \mV^{-\sfrac{1}{2}}_{n-1}}\right) = \log\left(\det{\mI + \sum^{\widehat{T}_n}_{k=1}\vm_k \vm^{\top}_k}\right).
    \end{align*}
    The matrix $\mM = \sum^{\widehat{T}_n}_{k=1}\vm_k \vm^{\top}_k$ by definition is positive semi-definite. Moreover, $\det{\mI + \mM} = \prod_{i\geq 1} (1 + \alpha_i)$, where $\alpha_i \geq 0$ are the eigenvalues of $\mM$. Furthermore, since $\alpha_i \geq 0$ and $\prod_{i\geq 1} (1 + \alpha_i) = 1 + \sum_{i\geq 1} \alpha_i + \cdots + \prod_{i\geq 1} \alpha_i$, we get $\prod_{i\geq 1} (1 + \alpha_i) \geq  1 + \sum_{i\geq 1} \alpha_i$.  Finally, since $\sum_{i\geq 1} \alpha_i = \tr(\mM)$, we get $\det{\mI + \mM} \geq 1 + \tr(\mM)$.
\end{proof}

\begin{corollary}[Upper bound on the posterior log determinant]
    \begin{align*}
        \log\left(\det{\mV_{n}}\right)& \leq \log\left(\det{\mV_{n-1}}\right) + \sum^{\widehat{T}_n}_{k=1}\sum^{d_x}_{j=1}\log\left(1 + \sigma^{-2} \sigma^2_{n-1, j}(\vz_{k, n})\right)
    \end{align*}
    \label{cor: upper bound on the posterior log det}
\end{corollary}

\begin{proof}
    \begin{align*}
        &\log\left(\det{\mV_{n}}\right) =\log\left(\det{\mV_{n-1}}\right) + \log\left(\det{\mI + \mM}\right) \\
        &\leq \log\left(\det{\mV_{n-1}}\right) + \log\left(\det{\diag\left(\mI + \mM\right)}\right) \tag{Hadamard's inequality for PSD matrices} \\
        &= \log\left(\det{\mV_{n-1}}\right) +  \sum^{\widehat{T}_n}_{k=1}\sum^{d_x}_{j=1}\log\left(1 + \sigma^{-2} \sigma^2_{n-1, j}(\vz_{k, n})\right)
    \end{align*}
\end{proof}

\Cref{cor: recursive info gain bound} will be useful for the discounted horizon case, whereas \cref{cor: upper bound on the posterior log det} will be applied for the nonepisodic setting.

Next, we show that \ombrl also performs optimism in the discounted horizon case.
\begin{lemma}
Let \cref{ass:lipschitz_continuity}, and \cref{ass:rkhs_func} hold. 
Consider the following definitions
\begin{align*}
    &J_{\gamma}(\vpi, \vf^*) = \E\left[\sum_{t=0}^{\infty} \gamma^{t} r(\vx_t, \vpi(\vx_{t}))\right] \; \text{s.t., }  \vx_{t+1} = \vf^*(\vx_{t}, \vpi(\vx_{t})) + \vw_t, \quad \vx_0 = \vx(0), \\
    &J_{\gamma}(\vpi, \vmu_n) = \E\left[\sum_{t=0}^{\infty} \gamma^{t}  r(\vx'_t, \vpi(\vx'_{t}))\right] \; \text{s.t., } \vx'_{t+1} = \vmu_n(\vx'_{t}, \vpi(\vx'_{t})) + \vw_t, \quad \vx'_0 = \vx(0), \\
    &\Sigma^{\gamma}_n(\vpi, \vf^*) = \E\left[\sum_{t=0}^{\infty} \gamma^{t} \norm{\vsigma_n(\vx_t, \vpi(\vx_{t}))}\right] \; \text{s.t., }  \vx_{t+1} = \vf^*(\vx_{t}, \vpi(\vx_{t})) + \vw_t, \quad \vx_0 = \vx(0), \\
    &\Sigma^{\gamma}_n(\vpi, \vmu_n) = \E\left[\sum_{t=0}^{\infty} \gamma^{t} \norm{\vsigma_n(\vx'_t, \vpi(\vx'_{t}))}\right] \; \text{s.t., } \vx'_{t+1} = \vmu_n(\vx'_{t}, \vpi(\vx'_{t})) + \vw_t, \quad \vx'_0 = \vx(0), \\
    &\lambda_n = C_{\max} \frac{\gamma}{1-\gamma} \frac{(1 + \sqrt{d_x}) \beta_{n-1}(\delta)}{\sigma},
    \end{align*}
    where $C_{\max} = \max\{R_{\max}, \sigma_{\max}\}$.
Then we have for all $n \geq 0$, $\vpi \in \Pi$ with probability at least $1-\delta$
\begin{align*}
     |J_{\gamma}(\vpi, \vf^*) - J_{\gamma}(\vpi, \vmu_n)| &\leq \lambda_n \Sigma^{\gamma}_n(\vpi, \vmu_n) \\
    |J_{\gamma}(\vpi, \vf^*) - J_{\gamma}(\vpi, \vmu_n)| &\leq \lambda_n \Sigma^{\gamma}_n(\vpi, \vf^*) 
\end{align*}
\label{lemma: Main Lemma Gaussian noise discounted case}
\end{lemma}

\begin{proof}
    We give the proof for $|J_{\gamma}(\vpi, \vf^*) - J_{\gamma}(\vpi, \vmu_n)| \leq \lambda_n(L_r, \vmu_n) \Sigma^{\gamma}_n(\vpi, \vmu_n)$. The same argument holds for the second inequality.  
    We can extend the result from \citet[Corollary 2.,]{sukhija2024maxinforl} to the discounted case and get
    \begin{align*}
        &J_{\gamma}(\vpi, \vmu_n) - J_{\gamma}(\vpi, \vf^*) = \E_{\vmu_n}\left[\sum^{\infty}_{t=0} \gamma^{t+1} (J_{\gamma}(\vpi, \vf^*, \vx'_{t+1}) -  J_{\gamma}(\vpi, \vf^*, \hat{\vx}_{t+1}))\right], \\
  &\text{with } \hat{\vx}_{t+1} = \vf^*(\vs'_{t}, \vpi(\vx'_t)) + \vw_t, \; \text{and } \vx'_{t+1} = \vmu_n(\vs'_{t}, \vpi(\vx'_t))+ \vw_t.
    \end{align*}
    Let $\beta_n\frac{1 + \sqrt{d_x}}{\sigma} C(\vx) = J_{\gamma}^2(\vpi, \vf^*, \vx)$. Note that $C(\vx) \leq \lambda_n$ for all $\vc \in \setX$. Therefore, we have
    \begin{align*}
        &|J_{\gamma}(\vpi, \vmu_n) - J_{\gamma}(\vpi, \vf^*)| \\
        &=  
 \left|\E\left[\sum^{\infty}_{t=0} \gamma^{t+1} (J_{\gamma}(\vpi, \vf^*, \vx'_{t+1}) -  J_{\gamma}(\vpi, \vf^*, \hat{\vx}_{t+1}))\right]\right| \\
 &\leq \sum^{\infty}_{t=0} \gamma^{t+1}\E\left[\left|\E_{\vw_t}\left[J_{\gamma}(\vpi, \vf^*, \vx'_{t+1}) -  J_{\gamma}(\vpi, \vf^*, \hat{\vx}_{t+1})\right]\right|\right] \\
&\leq \sum^{\infty}_{t=0} \gamma\E\left[\sqrt{\max\left\{\E_{\vw_t}[C(\vx'_{t+1})],  \E_{\vw_t}[C( \hat{\vx}_{t+1})]\right\}}\right.\\
&\times \left.\gamma^{t} \min\left\{\frac{\norm{\vf^*(\vx'_t, \vpi(\vx'_t)) - \vmu_n(\vx'_t, \vpi(\vx'_t))}}{\sigma}, 1\right\}\right] \tag*{ \citep[Lemma C.2.]{kakade2020information}} \\
&\leq \lambda_n \sum^{\infty}_{t=0} \E\left[\gamma^t\norm{\vsigma_{n-1}(\vx'_t, \vpi(\vx'_t))}\right] \tag{\citep[Corollary 3]{sukhija2024optimistic}}
    \end{align*}

    
\end{proof}

\begin{proof}[Proof of \cref{thm: theorem discounted setting}]
We start with bounding
\begin{equation}
   \sum^N_{n=1} \sum^{\infty}_{t=0} \E_{\vw_{1:t-1}}\left[\gamma^t\norm{\vsigma_{n-1}(\vx'_t, \vpi(\vx'_t))}^2\right]
\end{equation}

To achieve this, we use a sampling strategy where we increase the horizon of rollouts with each episode $n$. In the discounted setting, this allows us to collect data at the tails of our rollouts, i.e., make observations with longer rollouts and thus approximate the true discounted value function asymptotically.
Moreover, we set $T(n) = -\frac{\log(n)}{\log(\gamma)}$ (note that $\gamma < 1$ and therefore $T(n)$ is positive).
\begin{align*}
    &\sum^N_{n=1} \sum^{\infty}_{t=0} \E_{\vw_{1:t-1}}\left[\gamma^t\norm{\vsigma_{n-1}(\vx'_t, \vpi(\vx'_t))}^2\right] \\
    &= \sum^N_{n=1} \sum^{T(n) - 1}_{t=0} \E_{\vw_{1:t-1}}\left[\gamma^t\norm{\vsigma_{n-1}(\vx'_t, \vpi(\vx'_t))}^2\right] + \sum^N_{n=1} \sum^{\infty}_{t=T(n)} \E_{\vw_{1:t-1}}\left[\gamma^t\norm{\vsigma_{n-1}(\vx'_t, \vpi(\vx'_t))}^2\right] \\
    &\leq \sum^N_{n=1} \sum^{T(n) - 1}_{t=0} \E_{\vw_{1:t-1}}\left[\gamma^t\norm{\vsigma_{n-1}(\vx'_t, \vpi(\vx'_t))}^2\right] + \sum^N_{n=1} \gamma^{T(n)} \frac{\sigma^2_{\max}}{1-\gamma} \\
    &= \sum^N_{n=1} \sum^{T(n) - 1}_{t=0} \E_{\vw_{1:t-1}}\left[\gamma^t\norm{\vsigma_{n-1}(\vx'_t, \vpi(\vx'_t))}^2\right] + \sum^N_{n=1} n^{-1} \frac{\sigma^2_{\max}}{1-\gamma} \\
    &= \sum^N_{n=1} \sum^{T(n) - 1}_{t=0} \E_{\vw_{1:t-1}}\left[\gamma^t\norm{\vsigma_{n-1}(\vx'_t, \vpi(\vx'_t))}^2\right] + \frac{C\sigma^2_{\max}}{1-\gamma} \log(N)
\end{align*}

Next, we bound the term 
\begin{equation*}
    s_n =  \sum^{T(n) - 1}_{t=0} \gamma^t \sigma^{-2}\norm{\vsigma_{n-1}(\vx'_t, \vpi(\vx'_t))}^2.
\end{equation*}
Note that, $s_n \in  \left[0, \frac{\sigma^{-2}d_x\sigma^2_{\max}}{1-\gamma}\right)$. Let $s_{\max} = \frac{\sigma^{-2}d_x\sigma^2_{\max}}{1-\gamma}$, we have $s_n  \leq \frac{s_{\max}}{\log(1 + s_{\max})}\log(1 + s_n)$~\citep{srinivas}. Define $C_{\gamma} = \frac{s_{\max}}{\log(1 + s_{\max})}$. We have,
\begin{align*}
    s_n &\leq C_{\gamma}\log\left(1 + \sigma^{-2}\sum^{T(n) - 1}_{t=0} \gamma^t\norm{\vsigma_{n-1}(\vx'_t, \vpi(\vx'_t))}^2\right) \\
    &\leq C_{\gamma} \log\left(1 + \sigma^{-2}\sum^{T(n) - 1}_{t=0} \norm{\vsigma_{n-1}(\vx'_t, \vpi(\vx'_t))}^2\right)
\end{align*}

Finally, we have
\begin{align*}
    \sum^{N}_{n=1} s_n &\leq C_{\gamma} \sum^{N}_{n=1} \log\left(1 + \sigma^{-2}\sum^{T(n) - 1}_{t=0} \norm{\vsigma_{n-1}(\vx'_t, \vpi(\vx'_t))}^2\right) \\
    &\leq C_{\gamma} \Gamma_{\sum^{N}_{n=1} T(n)} \tag{\cref{cor: recursive info gain bound}} \\
    &\leq C_{\gamma} \Gamma_{N\log(N)}
\end{align*}


     \begin{align*}
        R_N &= \sum^N_{n=1} r_n \\
        &\le \sum^N_{n=1}(\lambda^2_n + 2\lambda_n) \Sigma^{\gamma}_n(\vpi_n, \vf^*) \\
        &\leq (\lambda^2_N + 2\lambda_N) \sum^N_{n=1}\Sigma^{\gamma}_n(\vpi_n, \vf^*) \\
        &=(\lambda^2_N + 2\lambda_N) \sqrt{N}\sqrt{\sum^{N}_{n=1} (\Sigma^{\gamma}_n(\vpi_n, \vf^*)})^2 \\
        &\leq (2\lambda_N + \lambda^2_N) \sqrt{N}\times \sqrt{\sum^{N}_{n=1} \E\left[\left(\sum_{t=0}^{\infty} \gamma^t\norm{\vsigma_n(\vx_t, \vpi(\vx_{t}))}\right]\right)^2} \\
        &\leq (2\lambda_N + \lambda^2_N) \sqrt{N} \times\sqrt{\sum^{N}_{n=1} \E\left[\left(\sum_{t=0}^{\infty} \gamma^t\right)\left(\sum_{t=0}^{\infty} \gamma^t\norm{\vsigma^2_n(\vx_t, \vpi(\vx_{t}))}\right]\right)} \\
        &= (2\lambda_N + \lambda^2_N) \sqrt{\frac{N}{1-\gamma}} \times \sqrt{\sum^{N}_{n=1} \E\left[\sum_{t=0}^{\infty} \gamma^t\norm{\vsigma^2_n(\vx_t, \vpi(\vx_{t}))}\right]} \\
        &\leq (2\lambda_N + \lambda^2_N) \sqrt{\frac{C_{\gamma}N\Gamma_{N\log(N)}}{1-\gamma} + \frac{C\sigma^2_{\max}N\log(N)}{(1-\gamma)^2} }
    \end{align*}
    Since $\lambda_N \propto \sfrac{\beta_N}{1-\gamma}$, we get
    \begin{equation*}
        R_N \le \setO\left(\Gamma^{\sfrac{3}{2}}_{N\log(N)}\sqrt{N}\right)
    \end{equation*}
\end{proof}

\subsection{Analysis for the nonepisodic RL case}
In this section, we prove \cref{thm: theorem informal average reward setting}. First, we restate the bounded energy assumption from \citet{sukhija2024neorl}.
\begin{definition}[$\setK_{\infty}$-functions]
The function $\xi: \R_{\geq 0} \to \R_{\geq 0}$ is of class $\setK_{\infty}$, if it is continuous, strictly increasing, $\xi(0) = 0$ and $\xi(s) \to \infty$ for $s \to \infty$.
\end{definition}

\begin{assumption}[Policies with bounded energy]
We assume there exists $\kappa, \xi \in \setK_{\infty}$,
    positive constants $K, C_u, C_l$ with $C_u > C_l$,  and $\gamma \in (0, 1)$ such that for each $\vpi \in \Pi$ we have,
\begin{itemize}[leftmargin=*]
    \item[] \label{assumption: Stability} {\em Bounded energy:}
    There exists a Lyapunov function $V^{\vpi}: \setX \to [0, \infty)$ for which  $\forall \vx, \vx' \in \setX$,
    \begin{align*} |V^{\vpi}(\vx) - V^{\vpi}(\vx')| &\leq \kappa(\norm{\vx-\vx'}) \tag{uniform continuity}\\
C_l \xi(\norm{\vx}) &\leq V^{\vpi}(\vx) \leq     C_u \xi(\norm{\vx}) \tag{positive definiteness}\\
        \E_{\vx_+|\vx, \vpi}[V^{\vpi}(\vx_+)] &\leq \gamma V^{\vpi}(\vx) + K \tag{drift condition}
    \end{align*}
    where $\vx_+ = \vf^*(\vx, \vpi(\vx)) + \vw$.
    \item[] {\em Bounded norm of reward:}
    \begin{equation*}
        \sup_{\vx \in \setX} \frac{r(\vx, \vpi(\vx))}{1 + V^{\vpi}(\vx)} < \infty 
    \end{equation*}
     \item[] {\em Boundedness of the noise with respect to $\kappa$:}
     \begin{equation*}
\E_{\vw}\left[\kappa(\norm{\vw})\right] < \infty, \
\E_{\vw}\left[\kappa^2(\norm{\vw})\right] < \infty
    \end{equation*}
\end{itemize}
\label{ass:Policy class}
\end{assumption}
\citet{sukhija2024neorl} argue that this assumption is often satisfied  in practice. We refer the reader to \citet{sukhija2024neorl} for further details. Next, we make an assumption on the underlying system $\vf^*$.
\begin{assumption}[Continous closed-loop dynamics, and Gaussian noise.]
\label{ass:neorl_system}
The dynamics model $\vf^*$ and all $\vpi \in \Pi$ are continuous, and process noise is i.i.d. Gaussian with variance $\sigma^2$, i.e., $\vw_t \stackrel{\mathclap{i.i.d}}{\sim} \setN(\vzero, \sigma^2\mI)$.
\end{assumption}

An important quantity in the average reward setting is the bias
\begin{equation}
    B(\vpi, \vx_0) = \lim_{T \to \infty} \E_{\vpi} \left[ \sum^{T-1}_{t=0}  
 r(\vx_t, \vu_t) - J_{\text{avg}}(\vpi) \right].
\end{equation}

The Bellman equation for the average reward setting is given by

    \begin{equation}
   B(\vpi, \vx) + J_{\text{avg}}(\vpi) = r(\vx, \vpi(\vx)) + \E_{\vx_+}[B(\vpi, \vx_+)|\vx, \vpi]
   \label{eq: Bellman Equation Average Cost}
\end{equation}

\citet{sukhija2024neorl} show that under \cref{assumption: Stability} and \cref{ass:neorl_system} the average reward solution and the bias (c.f.,~\cref{eq: average reward}) are bounded. Moreover, they show that with \cref{ass:rkhs_func} the average reward and bias are bounded for all dynamics $\vf \in \setM_n \cap \setM_0$. 

\begin{lemma}
Let \cref{ass:rkhs_func}, \cref{assumption: Stability}, and \cref{ass:rkhs_func} hold. 
Consider the following definitions
\begin{align*}
    &J_{\text{avg}}(\vpi, \vf^*) = \lim_{T \to \infty} \frac{1}{T}\E\left[\sum_{t=0}^{T-1} r(\vx_t, \vpi(\vx_{t}))\right] \; \text{s.t., }  \vx_{t+1} = \vf^*(\vx_{t}, \vpi(\vx_{t})) + \vw_t, \quad \vx_0 = \vx(0), \\
    &J_{\text{avg}}(\vpi, \vf) = \lim_{T \to \infty} \frac{1}{T}\E\left[\sum_{t=0}^{T-1}  r(\vx'_t, \vpi(\vx'_{t}))\right] \; \text{s.t., } \vx'_{t+1} = \vf(\vx'_{t}, \vpi(\vx'_{t})) + \vw_t, \quad \vx'_0 = \vx(0), \\
    &\Sigma_n(\vpi, \vf^*) = \lim_{T \to \infty} \frac{1}{T}\E\left[\sum_{t=0}^{T-1} \norm{\vsigma_n(\vx_t, \vpi(\vx_{t}))}\right] \; \text{s.t., }  \vx_{t+1} = \vf^*(\vx_{t}, \vpi(\vx_{t})) + \vw_t, \quad \vx_0 = \vx(0), \\
    &\Sigma_n(\vpi, \vf) = \lim_{T \to \infty} \frac{1}{T}\E\left[\sum_{t=0}^{T-1} \norm{\vsigma_n(\vx'_t, \vpi(\vx'_{t}))}\right] \; \text{s.t., } \vx'_{t+1} = \vf(\vx'_{t}, \vpi(\vx'_{t})) + \vw_t, \quad \vx'_0 = \vx(0), \\
    &\lambda_n = D_4(\vx_0, \gamma, K) \beta_{n-1}(\delta),
    \end{align*}
and $D_4(\vx_0, \gamma, K)$ is defined as in \citet[Theorem 3.1]{sukhija2024neorl}, is
independent of n and increases with $\norm{\vx_0}, K$ and $\gamma^{-1}$ (see~\citet{sukhija2024neorl} for the exact dependence).
Then we have for all $n \geq 0$, $\vpi \in \Pi$, $\vf \in \setM_n \cap \setM_0$ with probability at least $1-\delta$
\begin{align*}
     |J_{\text{avg}}(\vpi, \vf^*) - J_{\text{avg}}(\vpi, \vf)| &\leq \lambda_n \Sigma_n(\vpi, \vf) \\
    |J_{\text{avg}}(\vpi, \vf^*) - J_{\text{avg}}(\vpi, \vf)| &\leq \lambda_n \Sigma_n(\vpi, \vf^*) 
\end{align*}
\label{lemma: Main Lemma Gaussian noise average reward}
\end{lemma}
\begin{proof}
    \begin{align*}
        |J_{\text{avg}}(\vpi, \vf) - J_{\text{avg}}(\vpi, \vf^*)|
        &= \left|\lim_{T \to \infty} \frac{1}{T}\E_{\vf}\left[\sum_{t=0}^{T-1} r(\vx'_t, \vpi(\vx'_{t})) - J_{\text{avg}}(\vpi, \vf^*)\right]\right| \\
        &= \left|\lim_{T \to \infty} \frac{1}{T}\E_{\vf}\left[\sum_{t=0}^{T-1} B(\vx'_t, \vpi(\vx'_{t})) - B(\hat{\vx}^{'}_{t+1}, \vpi(\hat{\vx}^{'}_{t+1}))\right]\right|  \\
        &\leq \lambda_n \Sigma_n(\vpi, \vf) \tag{1}
    \end{align*}
    In the second last equality, we used the Bellman equation for the average reward setting (\cref{eq: Bellman Equation Average Cost}), where $\hat{\vx}^{'}_{t+1}$ is the next state under the true dynamics $\vf^*$. 

   For the last inequality, \citet{sukhija2024neorl} bound the bias term with $\lambda_n$ in Section A.3, on pages $23 - 24$.

    We can use the same derivation to show that $|J_{\text{avg}}(\vpi, \vf) - J_{\text{avg}}(\vpi, \vf^*)| \leq \lambda_n \Sigma_n(\vpi, \vf^*)$.
\end{proof}

\ombrl in the non-episodic setting operates similarly to \textsc{NeoRL}~\citep{sukhija2024neorl}. In particular, we update our model and policy every $T_n$ step, where $T_n$ is defined as:
\begin{align}
\label{equation: definition of Hn}
    T_n &= \max\left(\widehat{T_n}, \frac{\ceil{\log\left(\sfrac{C_u}{C_l}\right)}}{\log\left(\sfrac{1}{\gamma}\right)}\right), \\
    \widehat{T_n} &= \argmax_{T \ge 1} T+1 \\
    &\text{s.t.} \sum^{T}_{k=1} \sum^{d_x}_{j=1}\log\left(1 + \sigma^{-2} \sigma^2_{n-1, j}(\vz_{k, n})\right) \le \log(2).
\end{align}
Effectively, we update our model and policy only once we have accumulated more than one bit of information, i.e., $\sum^{T}_{k=1} \sum^{d_x}_{j=1}\log\left(1 + \sigma^{-2} \sigma^2_{n-1, j}(\vz_{k, n})\right) > \log(2)$. With the updated model and model set $\setM_n$, we select \emph{any} dynamics in $\vf_n \in \setM_n \cap \setM_0$ and pick the policy with
\begin{equation}
    \vpi_n = \underset{\vpi \in \Pi}{\arg\max}\; J_{\text{avg}}(\vpi, \vf_n) + \lambda_n \Sigma_n(\vpi, \vf_n).
    \label{eq: sampling neorl}
\end{equation}
Note that \citep{sukhija2024neorl} require maximizing over the dynamics in $\setM_n \cap \setM_0$, whereas we do not. Moreover, while this optimization is generally intractable, for \ombrl, we can obtain $\vf$ using the quadratic program described in~\cref{eq: Lipschitz function optimization}. However, in practice, we just pick the mean model $\vmu_n \in \setM_n$. This practical modification is also made in \citet{sukhija2024neorl} where they optimize over dynamics in $\setM_n$ instead of $\setM_n \cap \setM_0$.

\begin{theorem}[Formal Theorem statement for informal \cref{thm: theorem informal average reward setting}]
        Define $R_N = \sum^N_{n=1} \E[ 
    J_{\text{avg}}(\vpi^*) - r(\vx_n, \vpi_n(\vx_n)
    ]$. Let \cref{ass:rkhs_func}, \cref{assumption: Stability}, and \cref{ass:neorl_system} hold. Then we have for all $N \geq 0$ with probability at least $1-\delta$
    \begin{equation*}
        R_N \leq \setO\left(\Gamma^{\sfrac{3}{2}}_N \sqrt{N}\right) 
    \end{equation*}
\end{theorem}

\begin{proof}
Let $E_N$ denote the number of episodes after $N$ interactions in the environment.
    \begin{align*}
        &R_N = \E\left[\sum^{E_N}_{n=1} \sum^{T_n-1}_{k=0}
        J_{\text{avg}}(\vpi^*) - r(\vx^{n}_k, \vpi_n(\vx^{n}_k))\right] \\
        &\leq \E\left[\sum^{E_N}_{n=1} \sum^{T_n-1}_{k=0} J_{\text{avg}}(\vpi^*, \vf_n) + \lambda_n \Sigma_n(\vpi^*, \vf_n) - r(\vz^{n}_k) \right] \tag{\cref{lemma: Main Lemma Gaussian noise average reward}}\\
        &\leq \E\left[\sum^{E_N}_{n=1} \sum^{T_n-1}_{k=0} J_{\text{avg}}(\vpi_n, \vf_n) + \lambda_n \Sigma_n(\vpi_n, \vf_n) - r(\vz^{n}_k)\right] \tag{\cref{eq: sampling neorl}}\\
        &\leq \E\left[\sum^{E_N}_{n=1} \sum^{T_n-1}_{k=0} J_{\text{avg}}(\vpi_n, \vf_n) - r(\vz^{n}_k)\right] \\ 
        &+ \lambda_N \E\left[\sum^{E_N}_{n=1} \sum^{T_n-1}_{k=0} \Sigma_n(\vpi_n, \vf_n)\right] \\
        &\leq \setO\left(\Gamma_N \sqrt{N}\right) + \E\left[\lambda_N \sum^{E_N}_{n=1} \sum^{T_n-1}_{k=0} \Sigma_n(\vpi_n, \vf)\right] \tag{Theorem 3.1~\citet{sukhija2024neorl}} \\
    \end{align*}
    Next, we focus on $\E\left[\lambda_N \sum^{E_N}_{n=1} \sum^{T_n-1}_{k=0} \Sigma_n(\vpi_n, \vf)\right]$
    \begin{align*}
&\E\left[\sum^{E_N}_{n=1} \sum^{T_n-1}_{k=0} \Sigma_n(\vpi_n, \vf)\right] \\
        &= \E\left[\sum^{E_N}_{n=1} \sum^{T_n-1}_{k=0} \norm{\vsigma_n(\vx_t, \vpi(\vx_{t}))}\right] \\
        &+ \E\left[\sum^{E_N}_{n=1} \sum^{T_n-1}_{k=0} \Sigma_n(\vpi_n, \vf) -\norm{\vsigma_n(\vx_t, \vpi(\vx_{t}))}\right] \\
        &\leq C\sqrt{\Gamma_N N} + \E\left[\sum^{E_N}_{n=1} \sum^{T_n-1}_{k=0} \Sigma_n(\vpi_n, \vf) -\norm{\vsigma_n(\vx_t, \vpi(\vx_{t}))}\right] \tag{Lemma A.1~\citet{sukhija2024neorl}}\\
        &\leq \sqrt{N \Gamma_N} +  \setO\left(\Gamma_N \sqrt{N}\right)\tag{\citet[Theorem 3.1]{sukhija2024neorl} with reward $\vsigma_n$}
    \end{align*}
    Therefore 
    \begin{align*}
       \E\left[\lambda_T \sum^{E_N}_{n=1} \sum^{T_n-1}_{k=0} \Sigma_n(\vpi_n, \vf)\right] &\leq \setO\left(\lambda_N  \Gamma_N \sqrt{N}\right) \\
       &\leq \setO\left(\Gamma^{\sfrac{3}{2}}_N \sqrt{N}\right) 
    \end{align*}
    In conclusion, 
    \begin{equation*}
        R_N \leq \setO\left(\Gamma^{\sfrac{3}{2}}_N \sqrt{N}\right) 
    \end{equation*}
\end{proof}

\subsection{Analysis for pure intrinsic exploration}
In the following, we derive a sample complexity bound for a pure intrinsic exploration algorithm. Thereby showing convergence for methods such~\citet{buisson2020actively}.
\begin{theorem}
Let \cref{ass:lipschitz_continuity} and \cref{ass:rkhs_func} hold. Consider \ombrl with extrinsic reward $r = 0$, i.e., 
\begin{align*}
\vpi_n &= \underset{\vpi \in \Pi}{\arg\max}\; \E_{\vpi} \left[ \sum^{T-1}_{t=0}\norm{\vsigma_n(\vx'_t, \vu_t)}\right], \; \vx'_{t+1} = \vmu_n(\vx'_t, \vu_t) + \vw_t \notag.
\end{align*}
Then we have $\forall N > 0$, with probability at least $1-\delta$
\begin{equation*}
\max_{\vpi \in \Pi} \E_{\vf^*}\left[\sum_{t=0}^{T-1} \norm{\vsigma_n(\vx_t, \vpi(\vx_{t}))}\right] \leq \setO\left(\sqrt{\frac{\Gamma^{3}_{N}}{N}}\right).
\end{equation*}
\label{thm: intrinsic greedy}
\end{theorem}
\begin{proof}
   Let $\Sigma^*_N = \max_{\vpi} \Sigma_N(\vpi, \vf^*)$ and $\vpi^*_N$ the corresponding policy.
   \begin{align*}
       \Sigma^*_N &\leq \frac{1}{N}\sum^N_{n=1} \Sigma^*_n \tag{monotoncity of the variance} \\
       &\leq \frac{1}{N}\sum^N_{n=1} (1 + \lambda_n) \Sigma_n(\vpi^*_n, \vmu_n) \tag{\cref{lemma: Main Lemma Gaussian noise}} \\
       &\leq \frac{1}{N}\sum^N_{n=1} (1 + \lambda_n) \Sigma_n(\vpi_n, \vmu_n) \tag{$\vpi_n$ is the maximizer for mean dynamics $\vmu_n$} \\
       &\leq \frac{1}{N}\sum^N_{n=1} (1 + \lambda_n)^2 \Sigma_n(\vpi_n, \vf^*) \tag{\cref{lemma: Main Lemma Gaussian noise}} \\
       &\leq (1 + \lambda_N)^2\frac{1}{N}\sum^N_{n=1} \Sigma_n(\vpi_n, \vf^*) \\
       &\leq (1 + \lambda_N)^2\frac{1}{\sqrt{N}}\sum^N_{n=1} \Sigma^2_n(\vpi_n, \vf^*) \\
       &\leq \setO\left(\sqrt{\frac{\Gamma^{3}_{N}}{N}}\right)
   \end{align*}
\end{proof}
Effectively, \cref{thm: intrinsic greedy} shows that pure intrinsic exploration reduces our model epistemic uncertainty with a rate of $\sqrt{\frac{\Gamma^{3}_{N}}{N}}$. To the best of our knowledge, we are the first to show this. Moreover, \citet{sukhija2024optimistic} derive a similar bound but their algorithm performs optimistic exploration from \cref{eq:optimistic plan expensive} in addition to maximizing the intrinsic rewards. Our result shows that the optimistic exploration is not necessary for this setting.

\subsection{Analysis for the finite horizon setting with sub-Gaussia noise}
In the following, we analyse the regret for the setting where the process noise $\vw$ is $\sigma$-sub Gaussian. 
\begin{assumption}
    The dynamics model $\vf^*$, reward $r$, and all $\vpi \in \Pi$ are $L_f$, $L_r$ and $L_{\vpi}$ Lipschitz, respectively. Furthermore, we assume that process noise is i.i.d. $\sigma$-sub Gaussian.
    \label{ass: Lipschitz}
\end{assumption}
We make the same assumptions as other works~\citep{curi2020efficient, sussex2022model} that study this setting. Moreover, Lipschitz continuity is a common assumption for nonlinear dynamics~\citep{khalil2015nonlinear} and is satisfied for many real-world systems.

\citet{curi2020efficient} provide a regret bound that depends exponentially on the horizon $T$, i.e., $R_N \in \setO\left(\sqrt{\Gamma^{T}_N N}\right)$. They obtain an exponential dependence because when planning optimistically, i.e., solving \cref{eq:optimistic plan expensive}, they consider all plausible dynamics, including those that are not Lipschitz continuous for all $n$. Solving \cref{eq:optimistic plan expensive} for only continuous dynamics is intractable. However, for \ombrl, as we do not maximize over the set of dynamics we can overcome this limitation. 

Moreover, since $\vf^*$ has bounded RKHS norm, i.e., $\norm{\vf^*}_{k} \le B$ (~\cref{ass:rkhs_func}).
From \cite{srinivas, chowdhury2017kernelized} follows that with probability $1 - \delta$ we have for every $n$:
\begin{align*}
    \norm{\vf^* - \vmu_n}_{k_n} \le \beta_n.
\end{align*}

For \ombrl, instead of planning with the mean, which in general might not be Lipschitz continuous for all $n$, 
we select a function $\vf_n$ that not only approximates the $\vf^*$ function well, i.e.,  $\norm{\vf^* - \vf_n}_{k_n} \le \beta_n$, but also its RKHS norm does not grow with $n$. To do that we propose to solve the following quadratic optimization problem:
\begin{align}
\label{eq: Lipschitz function optimization}
    \vf_n = &\argmin_{\vf \in \text{span}(k(\vx_1, \cdot), \ldots, k(\vx_n, \cdot))} \norm{\vf - \vmu_n}_{k_n} \\
    &\text{s.t.} \norm{\vf}_{k} \leq B \notag
\end{align}

\begin{figure}[ht]
    \centering
    \includegraphics[width=\linewidth]{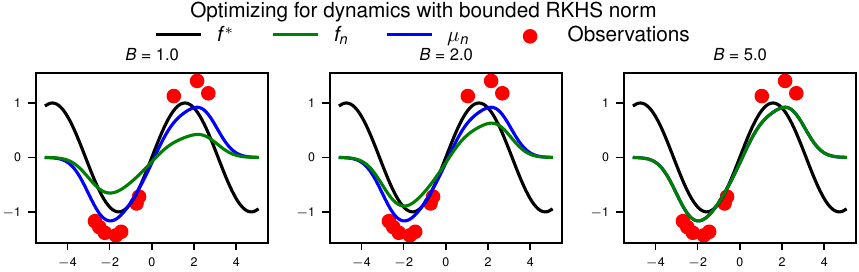}
    \caption{Solution to \cref{eq: Lipschitz function optimization} for different values for $B$. Effectively, for larger values for $B$, $\vmu_n$ and $\vf_n$ coincide.}
    \label{fig:function_optimization}
\end{figure}

\begin{theorem}
    The optimization problem \cref{eq: Lipschitz function optimization} is feasible and we have  $\norm{\vf_n - \vmu_n}_{k_n} \le \beta_n$.
\end{theorem}

\begin{proof}
    Consider the noise-free case, i.e., $\vw = 0$, and let $\bar{\vmu}_n$ posterior mean for this setting. For the function $\bar{\vmu}_n$ holds that $\norm{\vf^* - \bar{\vmu}_n}_{k_n} \le \beta_n$~(Corollary 3.11 of \citet{kanagawa2018gaussian}) and $\norm{\bar{\vmu}_n}_k \le B$ (Theorem 3.5 of \citet{kanagawa2018gaussian}). Since $\norm{\bar{\vmu}_n - \vmu_n}_{k_n} \le \norm{\bar{\vmu}_n - \vf^*}_{k_n} + \norm{\vf^* - \vmu_n}_{k_n} \le 2\beta_n$. By representer theorem, it also holds that $\bar{\vmu}_n \in \text{span}(k(\vz_1, \cdot), \ldots, k(\vz_n, \cdot))$.   
\end{proof}

Let $\valpha_n = (\mK + \sigma^2\mI)^{-1}\vy \in \R^n$ and reparametrize $\vf(\vx) = \sum_{i=1}^n\alpha_ik(\vx_i, \vx)$. We have $\norm{\vf}_{k}^2 = \valpha^\top \mK \valpha$. We also have:
\begin{align*}
    \norm{\vf - \vmu_n}_{k_n}^2 = (\valpha - \valpha_n)^\top \mK\left(\mI + \frac{1}{\sigma^2}\mK\right)(\valpha - \valpha_n)
\end{align*}
Hence the optimization problem \cref{eq: Lipschitz function optimization} is equivalent to:
\begin{align*}
    &\min_{\valpha \in \R^n} (\valpha - \valpha_n)^\top \mK\left(\mI + \frac{1}{\sigma^2}\mK\right)(\valpha - \valpha_n) \\
    &\text{s.t. } \valpha^\top \mK \valpha \le B^2
\end{align*}
This is a quadratic program that can be solved using any QP solver. The program finds the closest function to the posterior mean $\vmu_n$ that is Lipschitz continuous (see \cref{fig:function_optimization}). In particular, note that since $\norm{\vf_n}_k \leq B$, $\vf_n$ has a Lipschitz constant $L_{B}$ which is independent of $n$~\citep{berkenkamp2019safe}. From hereon, let $L_{*} = \max\{L_f, L_B\}$. 

For the sub-Gaussian case, \ombrl follows the same strategy as \cref{eq:optimistic plan} but instead of using the mean dynamics $\vmu_n$, we plan with the dynamics $\vf_n$ that are obtained from solving \cref{eq: Lipschitz function optimization}. 
\begin{align}
\vpi_n\! &=\! \arg\max_{\vpi \in \Pi}\;  \!\E_{\vpi} \!\!\left[ \sum^{T-1}_{t=0} r(\vx'_t, \vu_t) \!+\! \lambda_n \!\norm{\vsigma_n(\vx'_t, \vu_t)}\!\right]
  \label{eq:optimistic plan sub Gaussian} \\
  \vx'_{t+1}\!\! &= \!\vf_n(\vx'_t, \vu_t) + \vw_t \notag,
\end{align}

\begin{lemma}
Let \cref{ass: Lipschitz} and \cref{ass:rkhs_func} hold. 
Consider the following definitions
\begin{align*}
    &J(\vpi, \vf^*) = E\left[\sum_{t=0}^{T-1} r(\vx_t, \vpi(\vx_{t}))\right] \; \text{s.t., }  \vx_{t+1} = \vf^*(\vx_{t}, \vpi(\vx_{t})) + \vw_t, \quad \vx_0 = \vx(0), \\
    &J(\vpi, \vf_n) = E\left[\sum_{t=0}^{T-1} r(\vx'_t, \vpi(\vx'_{t}))\right] \; \text{s.t., } \vx'_{t+1} = \vf_n(\vx'_{t}, \vpi(\vx'_{t})) + \vw_t, \quad \vx'_0 = \vx(0), \\
    &\Sigma_n(\vpi, \vf^*) = E\left[\sum_{t=0}^{T-1} \norm{\vsigma_n(\vx_t, \vpi(\vx_{t}))}\right] \; \text{s.t., }  \vx_{t+1} = \vf^*(\vx_{t}, \vpi(\vx_{t})) + \vw_t, \quad \vx_0 = \vx(0), \\
    &\Sigma_n(\vpi, \vf_n) = E\left[\sum_{t=0}^{T-1} \norm{\vsigma_n(\vx'_t, \vpi(\vx'_{t}))}\right] \; \text{s.t., } \vx'_{t+1} = \vf_n(\vx'_{t}, \vpi(\vx'_{t})) + \vw_t, \quad \vx'_0 = \vx(0), \\
    &\lambda_n = (1 + d_x)L_r (1 + L_{\vpi})\bar{L}_{*}^{T-1} T \beta_n.
    \end{align*}
Then we have for all $n \geq 0$, $\vpi \in \Pi$ with probability at least $1-\delta$
\begin{align*}
     |J(\vpi, \vf^*) - J(\vpi, \vf_n)| &\leq \lambda_n \Sigma_n(\vpi, \vf_n) \\
    |J(\vpi, \vf^*) - J(\vpi, \vf_n)| &\leq \lambda_n\Sigma_n(\vpi, \vf^*) 
\end{align*}
\label{lemma: Main Lemma sub Gaussian}
\end{lemma}

\begin{proof}
\begin{align*}
    |J(\vpi, \vf^*) - J(\vpi, \vf_n)| 
    &= \E\left[\sum_{t=0}^{T-1} r(\vx_t, \vpi(\vx_{t})) - r(\vx'_t, \vpi(\vx'_{t}))\right] \leq L_r (1 + L_{\vpi}) \E\left[\sum_{t=0}^{T-1} \norm{\vx_t - \vx'_t}\right]
\end{align*}

Next we analyze $\norm{\vx_t - \vx'_t}$ for any $t$. Without loss of generality, assume $L_{*} \geq 1$ and define $\bar{L}_{*} = L_{*}(1 + L_{\vpi})$.

We show that 
\begin{align*}
    &\norm{\vx_{t+1} - \vx'_{t+1}} \leq (1 + \sqrt{d}_x) \beta_n \left(\sum^{t}_{k=0}\bar{L}^{t-k}_{*}\norm{\vsigma_n(\vx'_{k}, \vpi(\vx'_{k}))}\right).
\end{align*}

Consider $t=1$
\begin{align*}
    \norm{\vx_1 - \vx'_1} &=  \norm{\vf^*(\vx'_{0}, \vpi(\vx'_{0})) - \vf_n(\vx'_{0}, \vpi(\vx'_{0}))} \\
    &\leq
    (1 + \sqrt{d}_x) \beta_n \norm{\vsigma_n(\vx'_{0}, \vpi(\vx'_{0}))}
\end{align*}

Consider any $t > 1$,
\begin{align*}
    &\norm{\vx_{t+1} - \vx'_{t+1}}
    =  \norm{\vf^*(\vx_{t}, \vpi(\vx_{t})) - \vf_n(\vx'_{t}, \vpi(\vx'_{t}))} \\
    &\leq
    \norm{\vf^*(\vx'_{t}, \vpi(\vx'_{t})) - \vf_n(\vx'_{t}, \vpi(\vx'_{t}))} + \norm{\vf^*(\vx_{t}, \vpi(\vx_{t})) -\vf^*(\vx'_{t}, \vpi(\vx'_{t}))} \\
    &\leq (1 + \sqrt{d}_x) \beta_n \norm{\vsigma_n(\vx'_{t}, \vpi(\vx'_{t}))} + \bar{L}_{*} \norm{\vx_{t} - \vx'_{t}} \\
    &\leq (1 + \sqrt{d}_x) \beta_n \left(\norm{\vsigma_n(\vx'_{t}, \vpi(\vx'_{t}))}\right) + (1 + \sqrt{d}_x) \beta_n \left(\bar{L}_{*}\left(\sum^{t-1}_{k=0}\bar{L}^{t-1-k}_{*}\norm{\vsigma_n(\vx'_{k}, \vpi(\vx'_{k}))}\right) \right) \\
    &= (1 + \sqrt{d}_x) \beta_n \left(\sum^{t}_{k=0}\bar{L}^{t -k}_{*}\norm{\vsigma_n(\vx'_{k}, \vpi(\vx'_{k}))}\right)
\end{align*}

In particular, since $\bar{L}_{*} \geq 1$,
we have $\norm{\vx_{t+1} - \vx'_{t+1}} \leq (1 + \sqrt{d}_x) \beta_n \bar{L}^{t}_{*} \left(\sum^{t-1}_{k=0}\norm{\vsigma_n(\vx'_{k}, \vpi(\vx'_{k}))}\right)$.

In summary, we have
\begin{align*}
    &|J(\vpi, \vf^*) - J(\vpi, \vmu_n)|  = \E\left[\sum_{t=0}^{T-1} r(\vx_t, \vpi(\vx_{t})) - r(\vx'_t, \vpi(\vx'_{t}))\right] \\
    &\leq L_r (1 + L_{\vpi}) \E\left[\sum_{t=0}^{T-1} \norm{\vx_t - \vx'_t}\right] \\
    &\leq L_r (1 + L_{\vpi}) (1 + \sqrt{d}_x) \times \E\left[\sum_{t=0}^{T-1} \beta_n \bar{L}^{t-1}_{*} \left(\sum^{t-1}_{k=0}\norm{\vsigma_n(\vx'_{k}, \vpi(\vx'_{k}))}\right)\right] \\
    &\leq (1 + d_x)L_r (1 + L_{\vpi})\bar{L}^{T-1}_{*} T \beta_n \Sigma_n(\vpi, \vmu_n)\\
    &= \lambda_n \Sigma_n(\vpi, \vmu_n).
\end{align*}

\end{proof}

The main difference between our analysis and the analysis from \citet{curi2020efficient} is that for us $\lambda_n \propto \beta_n$ if we plan with $\vf_n$.

\begin{lemma}
    Let \cref{ass: Lipschitz} and \cref{ass:rkhs_func} hold and consider the simple regret at episode $n$, $r_n = J(\vpi^*, \vf^*) - J(\vpi_n, \vf^*)$. The following holds for all $n > 0$ with probability at least $1-\delta$
    \begin{equation*}
        r_n \leq  (2\lambda_n + \lambda^2_n)\Sigma_n(\vpi_n, \vf^*)
    \end{equation*}
    \label{lemma: simple regret sub Gaussian}
\end{lemma}
\begin{proof}
    \begin{align*}
        r_n &= J(\vpi^*, \vf^*) - J(\vpi_n, \vf^*) \\
        &\le J(\vpi^*, \vf_n) + \lambda_n \Sigma_n(\vpi^*, \vf_n) - J(\vpi_n, \vf^*) \tag{\cref{lemma: Main Lemma sub Gaussian}} \\
        &\leq J(\vpi_n, \vf_n) + \lambda_n \Sigma_n(\vpi_n, \vf_n) - J(\vpi_n, \vf^*) \tag{\cref{eq:optimistic plan sub Gaussian}} \\
        &= J(\vpi_n, \vf_n)  - J(\vpi_n, \vf^*)  + \lambda_n \Sigma_n(\vpi_n, \vf_n) \\
        &\leq \lambda_n \Sigma_n(\vpi_n, \vf^*) + \lambda_n \Sigma_n(\vpi_n, \vf_n) \tag{\cref{lemma: Main Lemma sub Gaussian}} \\
        &= 2\lambda_n \Sigma_n(\vpi_n, \vf^*) + \lambda_n (\Sigma_n(\vpi_n, \vf_n) - \Sigma_n(\vpi_n, \vf^*)) \\
        &\leq (\lambda^2_n + 2\lambda_n) \Sigma_n(\vpi_n, \vf^*).
    \end{align*}
    \end{proof}

\begin{theorem}[Finite horizon setting sub-Gaussian case]
Let \cref{ass: Lipschitz} and \cref{ass:rkhs_func} hold. Then we have $\forall N > 0$ with probability at least $1-\delta$
\begin{equation*}
    R_N \leq \setO\left(\Gamma^{\sfrac{3}{2}}_{N}\sqrt{N}\right).
\end{equation*}
\label{thm: finite horizon regret sub Gaussian}
\end{theorem}
\begin{proof}
    The proof is the same as for \cref{thm: finite horizon regret}, since in \cref{lemma: simple regret sub Gaussian} we show that also for the sub-Gaussian case, \ombrl has the same regret dependence w.r.t.~$\lambda_n$ and $\Sigma_n(\vpi_n, \vf^*)$.
\end{proof}
\clearpage
\section{Additional Experiments} \label{appendix: additional exps}
In this section, we provide additional experiments.
\begin{figure}[ht]
    \centering
    \includegraphics[width=\linewidth]{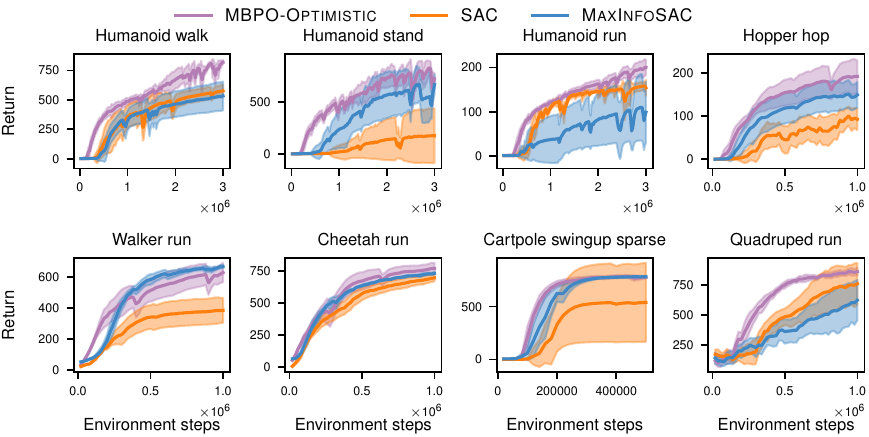}
    \caption{Comparison between \textsc{MBPO-Optimstic} and \textsc{MaxInfoSAC} and SAC. We observe that \textsc{MBPO-Optimstic}, being an MBRL algorithm, performs the best in terms of sample efficiency.}
    \label{fig:mbpo-mf}
\end{figure}

 \begin{figure*}[th]
    \centering
    \includegraphics[width=\linewidth]{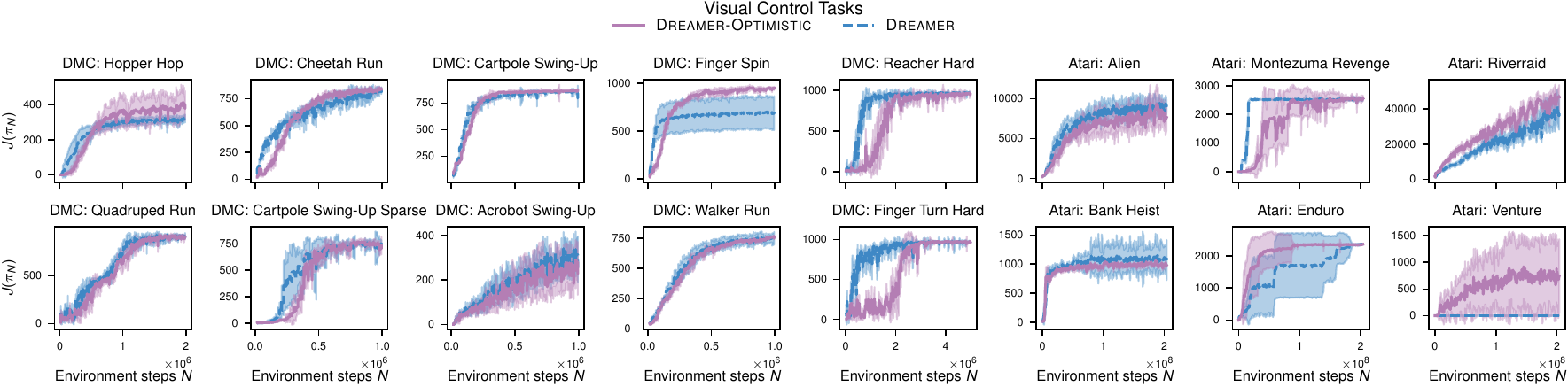}
    \caption{Learning curves for the visual control tasks from DMC and Atari using \textsc{Dreamer} as the base algorithm. \textsc{Dreamer-Optimistic} either performs on-par or better than \textsc{Dreamer} in all our experiments. Particularly, in the Venture task from the Atari benchmark, where \textsc{Dreamer} fails to obtain any rewards.}
            \label{fig:dreamer-visual-appendix}
\end{figure*}

\begin{figure}[ht]
    \centering
    \includegraphics[width=\linewidth]{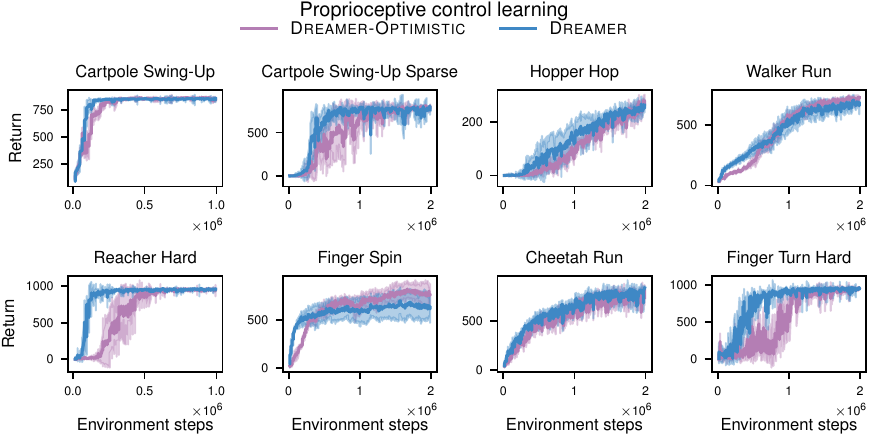}
    \caption{Experiments with \textsc{Dreamer-Optimistic} and \textsc{Dreamer} for proprioceptive tasks. \textsc{Dreamer-Optimistic} performs on par with \textsc{Dreamer}, obtaining slightly better performance on the Finger Spin task.}
    \label{fig:dreamer-proprio}
\end{figure}
\begin{figure}[ht]
    \centering
\includegraphics[width=0.5\textwidth]{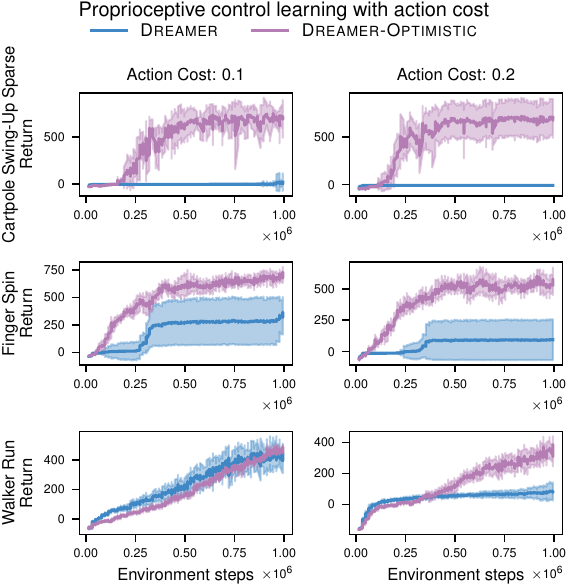}
    \caption{Experiments with \textsc{Dreamer-Optimistic} and \textsc{Dreamer} for proprioceptive tasks with action costs. \textsc{Dreamer} completely fails to solve the task, whereas \textsc{Dreamer-Optimistic} does not.}
    \label{fig:proprio-ac}
\end{figure}
\paragraph{Experiments with MBPO}
In \cref{fig:mbpo-mf} we compare \textsc{MBPO-Optimistic} with off-policy RL algorithms \textsc{MaxInfoSAC}~\citep{sukhija2024maxinforl} and SAC~\citep{haarnoja2018soft}.
From the figure, we conclude that \textsc{MBPO-Optimistic} performs the best in terms of sample-efficiency, particularly for the challenging/high-dimensional humanoid tasks. Moreover, between SAC and \textsc{MaxInfoSAC}, the latter achieves much better performance. We believe this is due to its intrinsic exploration reward. 
\paragraph{Experiments with \textsc{Dreamer}}
   
In \cref{fig:dreamer-visual-appendix} we compare \textsc{Dreamer-Optimistic} with \textsc{Dreamer} on additional environments. Overall, we observe that \textsc{Dreamer-Optimistic} performs either on par or better than \textsc{Dreamer}. However, for certain environments such as Reacher Hard or Finger Turn Hard, \textsc{Dreamer} is more sample-efficient. We believe this is because in these settings smaller values for $\lambda_n$ would suffice for exploration. However, we use a constant value for $\lambda_n$ across all environments and automatically update it using the approach proposed in \citet{sukhija2024maxinforl}. Investigating alternative strategies for $\lambda_n$, would generally benefit \ombrl methods. We think this is a promising direction for future work.

In \cref{fig:dreamer-proprio} and \cref{fig:proprio-ac} we compare \textsc{Dreamer-Optimistic} with  \textsc{Dreamer} on proprioceptive tasks. In most environments, \textsc{Dreamer-Optimistic} performs on par. It performs better in the Finger Spin environment. However, when action costs are introduced (\cref{fig:proprio-ac}), in line with our results in \cref{sec: experiments}, \textsc{Dreamer} fails to obtain any meaningful rewards.

\clearpage
\section{Experiment Details} \label{appendix: experiment_details}
In this section, we provide additional details for our experiments. 
\subsection{\textsc{MBPO-Optimistic}} For \textsc{MBPO-Optimistic}, we train an ensemble of forward dynamics models\footnote{For all tasks we use a $(256, 256)$ neural network architecture with $5$ ensembles, except for the humanoid and quadruped tasks where we use $(512, 512)$.}. We use the disagreement between the ensembles to quantify model epistemic uncertainty, similar to~\citet{pathak2019self, curi2020efficient, sukhija2024optimistic}. For selecting $\lambda_n$, we use the auto-tuning approach from \citet{sukhija2024maxinforl}, where the intrinsic reward weight is optimized by minimizing the following loss with stochastic gradient descent
\begin{equation}
   L(\lambda) =  \underset{\vx \sim \setD_{1:n}, \vu \sim \vpi_n, \bar{\vu} \sim \bar{\vpi}_n}{\E}\log(\lambda) (\vsigma_n(\vx, \vu) - \vsigma_n(\vx, \bar{\vu})).
   \label{eq: lambda opt}
\end{equation}  
Here $\bar{\vpi}_n$ is a target policy, which is updated using polyak updates of $\vpi_n$. This objective increases $\lambda$ when the policy is under exploring compared to the target policy. \citet{sukhija2024maxinforl} show that this strategy works across several model-free off-policy RL algorithms. 

Besides using the model to quantify disagreement, we generate additional data by adding the transitions predicted by our learned model. In particular, for every policy update, we sample a batch of transitions from the data buffer $(\vx, \vu, \vx') \sim \setD_{1:n}$, and add  $(\vx, \vu, \hat{\vx}')$, transitions predicted by our mean model $\vmu_n$, to the batch. This allows us to combine true rollouts with model generated rollouts, as proposed in \citet{janner2019trust}. Since we can generate additional data through our learned model, we can efficiently increase our update-to-data ratio (UTD). For all our experiments with MBPO, with use an UTD of $5$\footnote{We did not tune the UTD and chose $5$ to trade-off between computational cost and sample efficiency.}.

We use the same hyperparameters as \citet{sukhija2024maxinforl} for all our state-based experiments. 

\subsection{\textsc{Dreamer-Optimistic}}
We use \textsc{Dreamerv3} As the base model. For quantifying the model epistemic uncertainty, we use the same approach as \citet{sekar2020planning, mendonca2021discovering} and learn an ensemble of MLPs to model the latent dynamics\footnote{For all tasks we use a $(512, 512)$ neural network architecture with $5$ ensembles.}. The ensemble is only used for quantifying the model uncertainty/intrinsic reward. For the policy optimization, we use the \textsc{Dreamer} backbone, where the agent optimizes the policy using imagined rollouts.  For selecting $\lambda$, we also use the objective in \cref{eq: lambda opt}. We found adding a regularize term $\alpha * \abs{\lambda}$ to the objective worked better with \textsc{Dreamer}. We initialize $\lambda$ with $2$ and pick $\alpha = 0.001$. For the rest, we use the same hyperparemters as \textsc{Dreamer}\footnote{We use the $12$ million size model and the official \textsc{Dreamerv3} implementation (\url{https://github.com/danijar/dreamerv3/tree/main}).}.

\subsection{\textsc{SimFSVGD-Optimistic}}
\begin{figure}[H]
    \centering
\includegraphics[width=0.5\textwidth]{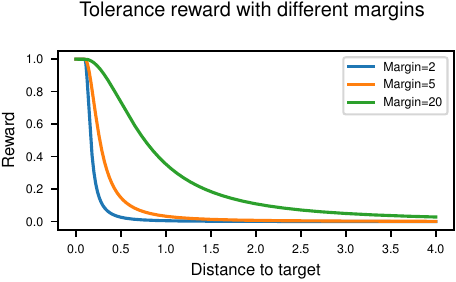}
    \caption{Tolerance reward function for different values of the margin. For larger margins, the agent receives rewards even if its further away from the target.}
    \label{fig:tolerance-reward}
\end{figure}
We use the same experiment setup, simulation prior, and hyperparameters as \citet{rothfuss2024bridging}\footnote{official implementation: \url{https://github.com/lasgroup/simulation_transfer}}. The reward function in \citet{rothfuss2024bridging} is based on the tolerance reward from \citet{tassa2018deepmind}. The tolerance function, gives higher rewards when the agent is close to a desired state, i.e., in case of the RC car the target position. The ``closesness'' is quantified using a margin parameter for the reward function. In \cref{fig:tolerance-reward} we plot the reward for different margin parameters. As we decrease the margin, the reward becomes sparser. \citet{rothfuss2024bridging} use a margin of $20$. In our simulation experiments, we show that \textsc{SimFSVGD} performs worse than \textsc{SimFSVGD-Optimistic} for smaller margins. For our hardware experiment, we use a margin of $5$, for which \textsc{SimFSVGD} fails to learn. For $\lambda_n$ we found that a linearly decaying schedule worked the best. Therefore, we linearly interpolated from $\lambda_0 = 0.5$ and $\lambda_{10} = 0$. After the tenth episode, the agent greedily maximized the extrinsic reward. 

\subsection{GP experiments}
For our GP experiments, we use the RBF kernel. The kernel parameters are updated online using maximum likelihood estimation~\citep{rasmussen2005gp}. For all the experiments, we use $\lambda_n = 10$ and for planning the iCEM optimizer~\citep{pinneri2021sample}. We use the same hyperparameters as \citet{sukhija2024neorl}\footnote{official implementation: \url{https://github.com/lasgroup/opax}}.

\subsection{Computational Costs}
\label{subsec:compute cost}
\begin{adjustbox}{max width=\linewidth}
\begin{threeparttable}
\centering
    \caption{Computation cost comparison for \ombrl with different base algorithms.}
    \label{tab:compute cost}
\begin{tabular}{l|l}
Algorithm & Training time \\ \hline \ \\
HUCRL (GPs)       & 90 +/- 3 min (Pendulum), \\
& 31.5 +/ 2.5 min (MountainCar)                     \\ \\
\ombrl (GPs)   & 30 +/- 0.6 min (Pendulum), \\ 
& 13.8 +/ 0.25 min (MountainCar)                               \\  \\  
\textsc{MBPO-Mean} & 9.6 +/- 0.2 min     \\
 (Time per 100k steps, 1 ensemble, GPU: NVIDIA GeForce RTX 2080 Ti) \\ \\
\textsc{MBPO-Optimistic} & 13.7 +/- 0.35 min                            \\   (Time per 100k steps, 5 ensembles, GPU: NVIDIA GeForce RTX 2080 Ti)   \\ \\  \textsc{Dreamer} & 42.24 +/- 0.95 min     \\ (Time per 100k steps, GPU: NVIDIA GeForce RTX 4090) \\ \\
\textsc{Dreamer-Optimistic} & 46.32 +/- 0.34 min  \\ (Time per 100k steps, 5 ensembles, GPU: NVIDIA GeForce RTX 4090)               
\end{tabular}
\end{threeparttable}
\end{adjustbox}

\newpage
\end{document}